\documentclass[11pt]{article}
\usepackage{url}
\usepackage{smile}
\usepackage{graphicx} 
\usepackage{subcaption}
\usepackage{algorithm}
\usepackage{algorithmic}
\usepackage{todonotes}
\usepackage{epstopdf}
\usepackage[colorlinks, linkcolor=blue, anchorcolor=blue, citecolor=blue]{hyperref}
\usepackage[margin=1in]{geometry}
\usepackage[normalem]{ulem}
\usepackage[export]{adjustbox}
\usepackage{mathtools, cuted}
\usepackage{natbib}
\usepackage{amsthm}
\usepackage{color}

\usepackage{todonotes}
\newcommand{\Note}[2]{}

\usepackage{kpfonts}
\DeclareMathAlphabet{\mathsf}{OT1}{cmss}{m}{n}
\SetMathAlphabet{\mathsf}{bold}{OT1}{cmss}{bx}{n}

\newcommand{\removed}[1]{}

\newcommand{\E}{\mathbb{E}}

\title{\huge \bf  Dimensionality Reduction for Stationary Time Series via Stochastic Nonconvex Optimization \footnote{Working in progress.}}

\author{Minshuo Chen, Lin Yang, Mengdi Wang, Tuo Zhao \thanks{Minshuo Chen and Tuo Zhao are affiliated with School of Industrial and Systems Engineering at Georgia Tech; Lin Yang and Mengdi Wang are affiliated with Department of Operations Research and Financial Engineering at Princeton University; Tuo Zhao is the corresponding author; Email:$\{$mchen393, tourzhao$\}$@gatech.edu.}}

\date{}

\begin{document}

\maketitle

\begin{abstract}
Stochastic optimization naturally arises in machine learning. Efficient algorithms with provable guarantees, however, are still largely missing, when the objective function is nonconvex and the data points are dependent. This paper studies this fundamental challenge through a streaming PCA problem for stationary time series data. Specifically, our goal is to estimate the principle component of time series data with respect to the covariance matrix of the stationary distribution. Computationally, we propose a variant of Oja's algorithm combined with downsampling to control the bias of the stochastic gradient caused by the data dependency. Theoretically, we quantify the uncertainty of our proposed stochastic algorithm based on diffusion approximations. This allows us to prove the asymptotic rate of convergence and further implies near optimal asymptotic sample complexity. Numerical experiments are provided to support our analysis.
\end{abstract}


\section{Introduction}\label{submission}

Many machine learning problems can be formulated as a stochastic optimization problem in the following form,
\begin{align}\label{online}
\min_{u}~\E_{Z\sim\cD} [f(u, Z)]~~~~\textrm{subject~to}~ u \in \cU,
\end{align}
where $f$ is a possibly nonconvex loss function, $Z$ denotes the random sample generated from some underlying distribution $\cD$ (also known as statistical model), $u$ is the parameter of our interest, and $\cU$ is a possibly nonconvex feasible set for imposing modeling constraints on $u$. For finite sample settings, we usually consider $n$ (possibly dependent) realizations of $Z$ denoted by $\{z_1,...,z_n\}$, and the loss function in \eqref{online} is further reduced to an additive form, $$\E [f(u, z)] = \frac{1}{n} \sum_{i=1}^n f(u, z_i).$$
For continuously differentiable $f$, \cite{robbins1951stochastic} propose a simple iterative stochastic search algorithm for solving \eqref{online}. Specifically, at the $k$-th iteration, we obtain $z_k$ sampled from $\cD$ and take
\begin{align}\label{projectedSG}
u_{k+1} = \Pi_{\cU} [u_k - \eta \nabla_u f(u_k, z_k)],
\end{align}
where $\eta$ is the step-size parameter (also known as the learning rate in machine learning literature), $\nabla_u f(u_k, z_k)$ is an unbiased stochastic gradient for approximating $\nabla_u \E_{Z\sim\cD} f(u_k, Z)$, i.e., $$\EE_{z_k}\nabla_u f(u_k, z_k) = \nabla_u \E_{Z\sim\cD} f(u_k, Z),$$ and $\Pi_{\cU}$ is a projection operator onto the feasible set $\cU$. This seminal work is the foundation of the research on stochastic optimization, and has a tremendous impact on the machine learning community. 

The theoretical properties of such a stochastic gradient descent (SGD) algorithm have been well studied for decades, when both $f$ and $\mathcal{U}$ are convex. For example, \cite{sacks1958asymptotic, bottou1998online, chung2004stochastic, shalev2011pegasos} show that under various regularity conditions, SGD converges to a global optimum as $k \rightarrow \infty$ at different rates. Such a line of research for convex and smooth objective function $f$ is fruitful and has been generalized to nonsmooth optimization \citep{duchi2012randomized, shamir2013stochastic, dang2015stochastic, reddi2016proximal}.

When $f$ is nonconvex, which appears more often in machine learning problems, however, the theoretical studies on SGD are very limited. The main reason behind is that the optimization landscape of nonconvex problems can be much more complicated than those of convex ones. Thus, conventional optimization research usually focuses on proving that SGD converges to first order optimal stationary solutions \citep{nemirovski2009robust}. More recently, some results in machine learning literature show that SGD actually converges to second order optimal stationary solutions, when the nonconvex optimization problem satisfies the so-called ``strict saddle property'' \citep{ge2015escaping, lee2017first}. More precisely, when the objective has negative curvatures at all saddle points, SGD can find a way to escape from these saddle points. A number of nonconvex optimization problems in machine learning and signal processing have been shown to satisfy this property, including principal component analysis (PCA), multiview learning, phase retrieval, matrix factorization, matrix sensing, matrix completion, complete dictionary learning, independent component analysis, and deep linear neural networks \citep{srebro2004linear, sun2015complete, ge2015escaping, sun2016geometric, li2016symmetry, ge2016matrix, chen2017online}. 

These results further motivate many followup works. For example, \cite{allen2017natasha} improves the iteration complexity of SGD from $\tilde{O}(\epsilon^{-4})$ in \cite{ge2015escaping} to $\tilde{O}(\epsilon^{-3.25})$ for general unconstrained functions, where $\epsilon$ is a pre-specifed optimization accuracy; \cite{jain2016streaming, allen2016first} show that the iteration complexity of SGD for solving the eigenvalue problem is $\tilde{O}(\epsilon^{-1})$. Despite of these progresses, we still lack systematic approaches for analyzing the algorithmic behavior of SGD. Moreover, these results focusing on the convergence properties, however, cannot precisely capture the uncertainty of SGD algorithms, which makes the theoretical analysis less intuitive.

Besides nonconvexity, data dependency is another important challenge arising in stochastic optimization for machine learning, since the samples $z_k$'s are often collected with a temporal pattern. For many applications (e.g., time series analysis), this may involve certain dependency. Taking generalized vector autoregressive (GVAR) data as an example, our observed $z_{k+1}\in\RR^m$ is generated by $$z^i_{k+1} | z_k \sim p(a_i^\top z_k),$$ where $a_i$'s are unknown coefficient vectors, $z^i_{k+1}$ is the $i$-th component of $z_{k+1}$, $p(\cdot)$ denotes the density of the exponential family, and $a_i^\top z_k$ is the natural parameter. Naturally, $\{z_k\}_{k=1}^\infty$ forms a Markov chain. There is only limited literature on convex stochastic optimization for dependent data. For example, \cite{duchi2012ergodic} investigate convex stochastic optimization algorithms for ergodic underlying data generating processes; \cite{homem2008rates} investigates convex stochastic optimization algorithms for dependent but identically distributed data. For nonconvex optimization problems in machine learning, however, how to address such dependency is still quite open.

This paper proposes to attack stochastic nonconvex optimization problems for dependent data by investigating a simple but fundamental problem in machine learning --- Streaming PCA for stationary time series. PCA has been well known as a powerful tool to reduce the dimensionality, and well applied to data visualization and representation learning. Specifically, we solve the following nonconvex problem,
\begin{align}\label{PCA}
U^* \in &\argmax_{U\in\RR^{m\times r}}~ {\rm Trace}(U^\top \Sigma U)~~~~\textrm{subject~to}~ U^\top U = I_r
\end{align}
where $\Sigma$ is the covariance matrix of our interest. This is also known as an eigenvalue problem. The column span of the optimal solution $U^*$ equals the subspace spanned by the eigenvectors corresponding to the first $r$ largest eigenvalues of $\Sigma$. Existing literature usually assumes that at the $k$-th iteration, we observe a random vector $z_k$ independently sampled from some distribution $\cD$ with $$\EE [z_k]=0~~\textrm{and}~~\EE [z_kz_k^\top]=\Sigma.$$ 
Our setting, however, assumes that $z_k$ is sampled from some time series with a stationary distribution satisfying
\begin{align*}
\lim_{k \rightarrow \infty}\EE [z_k]=0~~\textrm{and}~~\lim_{k\rightarrow \infty}\EE [z_kz_k^\top]=\Sigma. 
\end{align*}

There are two key computational challenges in such a streaming PCA problem:

\noindent $\bullet$ For time series, it is difficult to get unbiased estimators of the covariance matrix of the stationary distribution because of the data dependency. Taking GVAR as an example, the marginal distribution of $z_k$ is different from the stationary distribution. As a result, the stochastic gradient at the $k$-th iteration is biased, i.e., $$\EE[z_k z_k^\top U_k | U_k] \neq \Sigma U_k;$$

\noindent $\bullet$ The optimization problem in \eqref{PCA} is nonconvex, and its solution space is rotational-invariant. Given any orthogonal matrix $Q \in \mathbb{R}^{r \times r}$ and any feasible solution $U$, the product $UQ$ is also a feasible solution and gives the same column span as $U$. When $r > 1$, this fact leads to the degeneracy in the optimization landscape such that equivalent saddle points and optima are non-isolated. The algorithmic behavior under such degeneracy is still a quite open problem for SGD.

To address the first challenge, we propose a variant of Oja's algorithm to handle data dependency. Specifically, inspired by \cite{duchi2012ergodic}, we use downsampling to generate weakly dependent samples. Theoretically, we show that the downsampled data point yields a sequence of stochastic approximations of the covariance matrix of the stationary distribution with controllable small bias. Moreover, the block size for downsampling only logarithmically depends on the optimization accuracy, which is nearly constant (see more details in Sections \ref{downsample} and \ref{theory}).

To attack nonconvexity and the degeneracy of the solution space, we establish new convergence analysis based on principle angle between $U_k$ and the eigenspace of $\Sigma$. By applying diffusion approximations, we show that the solution trajectory weakly converges to the solution of a stochastic differential equation (SDE), which enables us to quantify the uncertainty of the proposed algorithm (see more details in Sections \ref{theory} and \ref{discussion}). Investigating the analytical solution of the SDE allows us to characterize the algorithmic behavior of SGD in three different scenarios: escaping from saddle points, traversing between stationary points, and converging to global optima. We prove that the stochastic algorithm asymptotically converges and achieves near optimal asymptotic sample complexity.

There are several closely related works. \cite{chen2017online} study the streaming PCA problem for $r=1$ also based on diffusion approximations. However, $r = 1$ makes problem \eqref{PCA} admit an isolated optimal solution, unique up to sign change. For $r > 1$, the global optima are nonisolated due to the rotational invariance property. Thus, the analysis is more involved and challenging. Moreover, \cite{jain2016streaming, allen2016first} provide nonasymptotic analysis for the Oja's algorithm for streaming PCA. Their techniques are quite different from ours. Their nonasymptotic results, though more rigorous in describing discrete algorithms, lack intuition and can only be applied to the Oja's algorithm with no data dependency. In contrast, our analysis handles data dependency and can be generalized to other stochastic optimization algorithms such as stochastic generalized Hebbian algorithm (see more details in Section \ref{theory}).


\noindent {\bf Notations:} Given a vector $v = (v_1, \dots, v_m)^\top \in \mathbb{R}^m$, we define the Euclidean norm $\lVert v \rVert_2^2 = v^\top v$. Given a matrix $A \in \mathbb{R}^{m \times n}$, we define the spectral norm $\lVert A \rVert_2$ as the largest singular value of $A$ and the Frobenius norm $\lVert A \rVert_{\textrm{F}}^2 = \textrm{Trace}(AA^\top)$. We also define $\sigma_r(A)$ as the $r$-th largest singular value of $A$. For a diagonal matrix $\Theta \in \mathbb{R}^{m \times m}$, we define $\sin \Theta = \textrm{diag}\left(\sin(\Theta_{11}), \dots, \sin(\Theta_{mm})\right)$ and $\cos \Theta = \textrm{diag}\left(\cos(\Theta_{11}), \dots, \cos(\Theta_{mm})\right)$. We denote the canonical basis of $\mathbb{R}^m$ by $e_i$ for $i = 1, \dots, m$ with the $i$-th element being 1, and the canonical basis of $\mathbb{R}^r$ by $e'_j$ for $j = 1, \dots, r$. We denote $x \asymp y$, meaning that $x$ and $y$ are asymptotically equal.


\section{Bias Control for SGD by Downsampling}\label{downsample}

This section devotes to constructing a nearly unbiased covariance estimator for the stationary distribution, which is crucial for our SGD algorithm. Before we proceed, we first briefly introduce geometric ergodicity for time series, which characterizes the mixing time of a Markov chain. 
\begin{definition}[Total Variation Distance]
Given two measures $\mu$ and $\nu$ on the same measurable space $(\Omega, \cF)$, the total variation distance is defined to be
\begin{align}
\cD_{\rm{TV}}(\mu, \nu) = \sup_{A \in \cF} \left|\mu(A) - \nu(A)\right|. \notag
\end{align}
\end{definition}
\begin{definition}[Geometric Ergodicity]\label{geoergodic}
A Markov chain with state space $S$ and stationary distribution $\pi$ is geometrically ergodic, if it is positive recurrent and there exists an absolute constant $\rho \in (0, 1)$ such that
\begin{align}
\cD_{\textrm{TV}}\left(p^n(x, \cdot), \pi(\cdot)\right) = O\left(\rho^n\right) ~\textrm{for all}~  x \in S, \notag
\end{align}
where $p^n(\cdot, \cdot)$ is the $n$-step transition kernel.
\end{definition}
Note that $\rho$ is independent of $n$ and only depends on the underlying transition kernel of the Markov chain. The geometric ergodicity is equivalent to saying that the chain is $\beta$-mixing with an exponentially decaying coefficient \citep{bradley2005basic}.

As aforementioned, one key challenge of solving the streaming PCA problem for time series is that it is difficult to get unbiased estimators of the covariance matrix $\Sigma$ of the stationary distribution. However, when the time series is geometrically ergodic, the transition probability $p^h(z_k, z_{k+h})$ converges exponentially fast to the stationary distribution. This allows us to construct a nearly unbiased estimator of $\Sigma$ as shown in the next lemma.

\begin{lemma}\label{unbiaslemma}
Let $\{z_k\}_{k=1}^\infty$ be a geometrically ergodic Markov chain with parameter $\rho$, and assume $z_k$ is Sub-Gaussian. Given a pre-specified accuracy $\tau$, there exists $$h = O\left(\kappa_\rho \log \frac{1}{\tau}\right)$$ such that we have $$\E \left [\frac{1}{2}(z_{2h+k} - z_{h+k})(z_{2h+k}-z_{h+k})^\top \Big | z_k \right ] = \Sigma + E \Sigma$$ with $\lVert E \rVert_2 \leq \tau$, where $\kappa_\rho$ is a constant depending on $\rho$.
\end{lemma}

Lemma \ref{unbiaslemma} shows that as $h$ increases, the bias decreases to zero. This suggests that we can use the downsampling method to reduce the bias of the stochastic gradient. Specifically, we divide the data points into blocks of length $2h$ as shown below.
$$\underbrace{z_1, z_2, \dots, z_{2h}}_{\textrm{the 1-st block}}, \underbrace{z_{2h+1}, \dots, z_{4h}}_{\textrm{the 2-nd block}}, \dots, \underbrace{z_{2(b-1)h+1}, \dots, z_{2bh}}_{\textrm{the b-th block}}$$
For the $s$-th block, we use data points $z_{(2s-1)h}$ and $z_{2sh}$ to approximate $\Sigma$ by $$X_s = \frac{1}{2}(z_{2sh} - z_{(2s-1)h}) (z_{2sh} - z_{(2s-1)h})^\top.$$
Later we will show that the block size $h$ only needs to be the logarithm of the optimization accuracy, which is nearly constant. Thus, the downsampling is affordable. Moreover, if the stationary distribution has zero mean, we only need the block size to be $h$ and $X_s = z_{sh}z_{sh}^\top$. 

Many popular time series models in machine learning are geometrically ergodic. Here we discuss a few examples.
\begin{example}
The vector autoregressive (VAR) model follows the update $$z_{k+1} = A z_k + \epsilon_k,$$ where $\epsilon_k$'s are i.i.d. Sub-Gaussian random vectors with $\EE[\epsilon_k] = 0~\textrm{and}~\EE[\epsilon_k \epsilon_k^\top] = \Gamma,$ and $A$ is the coefficient matrix. When $\rho = \lVert A \rVert_2 < 1$, the model is stationary and geometrically ergodic \citep{tjostheim1990non}. Moreover, the mean of its stationary distribution is 0.
\end{example}

\begin{example}
Recall that GVAR model follows $$z^i_{k+1} | z_k \sim p(a_i^\top z_k),$$ where $z_{k+1}^i$'s are independent conditioning on $z_k$. The density function is $p(x | \theta) = h(x) \exp\left(T(x)\theta - B(\theta)\right),$ where $T(x)$ is a statistic, and $B(\theta)$ is the log partition function. GVAR is stationary and geometrically ergodic under certain regularity conditions \citep{hall2016inference}.
\end{example}
\begin{example}
Gaussian Copula VAR model assumes there exists a latent Gaussian VAR skeleton, i.e., $$w_{k+1} = A w_k + \epsilon_k$$ with $\epsilon_k$ being i.i.d. Gaussian. The observation $$z^i_k = f_i(w_k^i)$$ is a monotone transformation of $w_k^i$. \cite{han2013principal} construct a sequence of rank-based transformed Kendall's tau covariance estimators $\{\hat{\Sigma}_k\}_{k=1}^\infty$ for the stationary covariance of $w_k$, and show that $\{\hat{\Sigma}_k\}_{k=1}^\infty$ is stationary and geometrically ergodic.
\end{example}

As an illustrative example, we show that for Gaussian VAR with $\rho = \lVert A \rVert_2 < 1~\textrm{and}~\Gamma = I,$ the bias of the covariance estimator can be controlled by choosing $h = O\left(\frac{1}{1 - \rho} \log \frac{1}{\tau}\right)$. The covariance matrix of the stationary distribution is $\Sigma = \sum_{i=0}^\infty A^i (A^\top)^i$. One can check
\begin{align}
\E \left [z_{h+k}z_{h+k}^\top | z_k \right ] - \Sigma & = \E \left [A^h z_k z_k^\top (A^\top)^h | z_k \right ] + \E \left [\sum_{i=0}^{h-1} A^i \epsilon_k \epsilon_k^\top (A^\top)^i \bigg | z_k \right ] - \Sigma \notag \\
& = \underbrace{A^h z_k z_k^\top (A^\top)^h}_{T_1} + \underbrace{\sum_{i=h}^{\infty} A^i (A^\top)^i}_{T_2}\notag.
\end{align}

Here the spectrum of $A$ acts as the geometrically decaying factor for both $T_1$ and $T_2$, since $$\left \lVert T_1 \right \rVert_2 = O(\rho^{2h})~~\textrm{and}~~\left \lVert T_2 \right \rVert_2 = O(\rho^{2h}).$$ As a result, the bias of $\E\left [z_{h+k}z_{h+k}^\top | z_k \right]$ decays to zero exponentially fast. We pick $$h = O\left(\frac{1}{1 - \rho} \log \frac{1}{\tau}\right),$$ and obtain $\E \left [z_{k+h}z_{k+h}^\top | z_k \right ] = \Sigma + E \Sigma \textrm{~~with~~} \lVert E \rVert_2 \leq \tau.$


\section{Downsampled Oja's Algorithm}

We introduce a variant of Oja's algorithm combined with our downsampling technique. For simplicity, we assume the stationary distribution has mean zero. We summarize the algorithm in Algorithm \ref{ojaalg}.
\begin{algorithm}[h]
   \caption{Downsampled Oja's Algorithm}\label{ojaalg}
   \label{Oja}
\begin{algorithmic}
   \STATE {\bfseries Input:} data points $z_k$, block size $h$, step size $\eta$.
   \STATE Initialize $U_1$ with orthonormal columns.
   \STATE Set $s \gets 1$
   \REPEAT
   \STATE Take sample $z_{sh}$, and set $X_s \gets z_{sh} z_{sh}^\top$.
   \STATE $U_{s+1} \gets \Pi_{\textrm{Orth}} (U_s + \eta X_s U_s)$.
   \STATE $s \gets s + 1$.
   \UNTIL Convergence
   \STATE {\bfseries Output:} $U_s$
\end{algorithmic}
\end{algorithm}

The projection $\Pi_{\textrm{Orth}}(U)$ denotes the orthogonalization operator that performs on columns of $U$. Specifically, for $U \in \mathbb{R}^{m \times r}$, $\Pi_{\textrm{Orth}}(U)$ returns a matrix $U' \in \mathbb{R}^{m \times r}$ that has orthonormal columns. Typical examples of such operators include Gram-Schmidt method and Householder transformation. The step $$U_{s+1} = \Pi_{\textrm{Orth}} (U_s + \eta X_s U_s)$$ is essentially the original Oja's update. Our variant manipulates on data points by downsampling such that $X_s$ is nearly unbiased. We emphasize that $s$ denotes the number of iterations, and $k$ denotes the number of samples.


\section{Theory}\label{theory}

Before we proceed, we impose some model assumptions on the problem.
\begin{assumption}\label{assump1}
There exists an eigengap in the covariance matrix $\Sigma$ of the stationary distribution, i.e., $$\lambda_1 \geq \dots \geq \lambda_r > \lambda_{r+1} \geq \dots \geq \lambda_m > 0,$$ where $\lambda_i$ is the $i$-th eigenvalue of $\Sigma$.
\end{assumption}

\begin{assumption}\label{assump2}
Data points $\{z_k\}_{k \geq 1}$ are generated from a geometrically ergodic time series with parameter $\rho$, and the stationary distribution has mean zero. Each $z_k$ is Sub-Gaussian, and the block size is chosen as $h = O\left(\kappa_\rho \log\frac{1}{\eta}\right)$ for downsampling.
\end{assumption}

The eigengap in Assumption \ref{assump1} requires the optimal solution is identifiable. Specifically, the optimal solution $U^*$ is unique up to rotation. The positive definite assumption on $\Sigma$ is for theoretic simplicity, however, it can be dropped as discussed in Section \ref{discussion}. Assumption \ref{assump2} implies that each $z_k$ has bounded moments of any order. 

We also briefly explain the optimization landscape of streaming PCA problems as follows. Specifically, we consider the eigenvalue decomposition $$\Sigma = R \Lambda R^\top\quad\textrm{with}\quad\Lambda = \textrm{diag}(\lambda_1, \lambda_2, \dots, \lambda_m).$$ Recall that $e_i$ is the canonical basis of $\mathbb{R}^m$. If $U$ is a global maximum, the column span of $R^\top U$ equals the subspace spanned by $\{e_1, \dots, e_r\}$. If we replace any one or more than one of the $e_i$'s for $i \in \{1, \dots, r\}$ with $e_j$'s for $j \in \{r+1, \dots, m\}$, then $U$ becomes a saddle point or a global minimum. When $U$ is a stationary point, we denote the column span of $R^\top U$ by the span of $\{e_{a_1}, \dots, e_{a_r}\}$, where $\cA_r = \{a_1, \dots, a_r\} \subset \{1, \dots, m\}$. For convenience, we say that $U$ is a stationary point corresponding to the set $\cA_r$.

To handle the rotational invariance of the solution space, we use principle angle to characterize the distance between the column span of $U^*$ and $U_s$.
\begin{definition}[Principle Angle]
Given two matrices $U \in \mathbb{R}^{m \times r_1}$ and $V \in \mathbb{R}^{m \times r_2}$ with orthonormal columns, where $1 \leq r_1 \leq r_2 \leq m$, the principle angle between these two matrices is,
\begin{align}
\Theta(U, V) = \textrm{diag}\Big(\arccos \left(\sigma_1(U^\top V)\right), \dots, \arccos\left(\sigma_{r_1}(U^\top V)\right)\Big). \notag
\end{align}
\end{definition}

We show the consequence of using principle angle as follows. Specifically, any optimal solution $U^*$ satisfies $$\left \lVert \sin \Theta(R_r, U^*) \right \rVert_{\textrm{F}}^2 = \left \lVert \cos \Theta(\overline{R}_r, U^*) \right \rVert_{\textrm{F}}^2 = 0,$$ where $R_r$ denotes the first $r$ columns of $R$, and $\overline{R}_r$ denotes the last $m-r$ columns of $R$. This essentially implies that the column span of $U^*$ is orthogonal to that of $\overline{R}_r$. Thus, to prove the convergence of SGD, we only need to show 
\begin{align}\label{sgdconverge}
\left \lVert \cos \Theta(\overline{R}_r, U_s) \right \rVert_{\textrm{F}}^2 \rightarrow 0.
\end{align} 
By the rotational invariance of principle angle, we obtain $$\Theta\left(\overline{R}_r, U_s\right) = \Theta\left(R^\top \overline{R}_r, R^\top U_s\right) = \Theta\left(\overline{E}_r, R^\top U_s\right),$$ where $\overline{E}_r = [e_{r+1}, \dots, e_m]$. For notational simplicity, we denote $\overline{U}_s = R^\top U_s$. Then \eqref{sgdconverge} is equivalent to $$\left \lVert \cos \Theta\left(\overline{E}_r, \overline{U}_s\right) \right \rVert_{\textrm{F}}^2 \rightarrow 0.$$ We need such an orthogonal transformation, because $\left \lVert \cos \Theta\left(\overline{E}_r, \overline{U}_s\right) \right \rVert_{\textrm{F}}^2$ can be expressed as $$\left \lVert \cos \Theta\left(\overline{E}_r, \overline{U}_s\right) \right \rVert_{\textrm{F}}^2 = \sum_{i=r+1}^m \left \lVert e_i^\top \overline{U}_s \right \rVert_2^2 = \sum_{i=r+1}^m \gamma_{i, s}^2,$$ where $\gamma^2_{i, s} = \left \lVert e_i^\top \overline{U}_s \right \rVert_2^2$. 

\subsection{Global Convergence by ODE}

One can check that the sequence $\left \{\left(z_{sh}, \overline{U}_s\right)\right \}_{s=1}^\infty$ forms a discrete Markov process. We apply diffusion approximations to establish global convergence of SGD. Specifically, by a continuous time interpolation, we construct continuous time processes $U^\eta(t)$ and $X^\eta(t)$ such that $$U^\eta(t) = U_{\lfloor t / \eta \rfloor + 1}~~\textrm{and}~~X^\eta(t) = X_{\lfloor t / \eta \rfloor + 1}.$$ Note that the subscript $\lfloor t / \eta \rfloor + 1$ denotes the number of iterations, and the superscript $\eta$ highlights the dependence on $\eta$. We also denote $$\overline{U}^\eta(t) = R^\top U^\eta(t)~~\textrm{and}~~\overline{X}^\eta(t) = R^\top X^\eta(t) R.$$ We denote the continuous time version of $\gamma_{i, s}^2$ by $$\gamma_{i, \eta}^2(t) =\left \lVert e_i^\top \overline{U}^\eta(t) \right \rVert_2^2.$$
It is difficult to directly characterize the global convergence of $\gamma_{i, \eta}^2(t)$. Thus we introduce an upper bound of $\gamma_{i, \eta}^2(t)$ as follows.
\begin{lemma}\label{upperbound}
Let $E_r = [e_1, \dots, e_r] \in \mathbb{R}^{m \times r}$. Suppose $\overline{U}^\eta(t)$ has orthonormal columns and $E_r^\top \overline{U}^\eta(t)$ is invertible. We have
\begin{align}\label{paupperbound}
\tilde{\gamma}_{i, \eta}^2(t) = \left \lVert e_i^\top \overline{U}^\eta(t) \left(E_r^\top \overline{U}^\eta(t)\right)^{-1} \right \rVert_2^2 \geq \gamma_{i, \eta}^2(t).
\end{align}
\end{lemma}
The detailed proof of Lemma \ref{upperbound} is provided in Appendix \ref{upperbound-proof}. We next show $\tilde{\gamma}_{i, \eta}^2(t)$ converges in the following theorem.

\begin{theorem}\label{stage2ode}
As $\eta \rightarrow 0$, the process $\tilde{\gamma}^2_{i, \eta}(t)$ weakly converges to the solution of the ODE
\begin{align}\label{ODE}
d\tilde{\gamma}_i^2 = b_i \tilde{\gamma}^2_i dt\quad  \textrm{with}\quad b_i \leq 2(\lambda_i - \lambda_r),
\end{align}
where $\tilde{\gamma}^2_i(0) = \left \lVert e_i^\top \overline{U}(0) \left(E_r^\top \overline{U}(0)\right)^{-1} \right \rVert_2^2$, and $\overline{U}(0)$ has orthonormal columns.
\end{theorem}
\begin{proof}[Proof Sketch]
Due to space limit, we only present a sketch. Our derivation is based on the Infinitesimal Generator Approach (IGA). The detailed proof is provided in Appendix \ref{proof4.5}. Specifically, as shown in \cite{dieci1999smooth}, the orthogonalization operator $\Pi_{\textrm{Orth}}(U)$ is twice differentiable, when $U$ is column full rank. Since $X_s$ is positive semidefinite and $U_s$ is initialized with orthonormal columns, $U_s + \eta X_s U_s$ is always guaranteed to be column full rank. Thus, we consider the second order Taylor approximation of $\Pi_{\textrm{Orth}}(U_s + \eta X_s U_s)$ as follows,
\begin{align}\label{Ojacont}
\overline{U}_{s+1} = \overline{U}_s & + \eta \left(I - \overline{U}_s \overline{U}_s^\top\right) \overline{X}_s \overline{U}_s + \eta^2 \overline{W},
\end{align}
where $\overline{W}$ is the remainder, and satisfies $\left \lVert \overline{W} \right \rVert_2 = O\left (\left \lVert \overline{X}_s \right \rVert_2 \right )$. We then show that given the increment, $$\Delta \tilde{\gamma}_{i, s}^2 = \tilde{\gamma}_{i, s+1}^2 - \tilde{\gamma}_{i, s}^2,$$ the infinitesimal conditional expectation and conditional variance satisfy
\begin{align}
& \lim_{\eta \rightarrow 0} \eta^{-1} \EE \left[{\Delta \tilde{\gamma}_{i, s}^2} \big| \overline{U}_s, z_{sh}\right ] = b_i \tilde{\gamma}_{i, s}^2,\label{odemean} \\
& \lim_{\eta \rightarrow 0} \eta^{-1} \EE \left[{\left [\Delta \tilde{\gamma}_{i, s}^2 \right]^2} \big| \overline{U}_s, z_{sh}\right] = 0.\label{odevariance}
\end{align}
Thus, $\tilde{\gamma}^2_{i, \eta}(t)$ weakly converges to \eqref{ODE} as shown in \cite{ethier2009markov}. Note that when taking expectation in \eqref{odemean} and \eqref{odevariance}, we need a truncation argument on the tail of $\overline{X}_s$.
\end{proof}

The analytical solution to \eqref{ODE} is
\begin{align*}
\tilde{\gamma}^2_i(t) = \tilde{\gamma}^2_i(0) e^{b_i t}.
\end{align*}
Thus, we have $$b_i \leq 2(\lambda_{r+1} - \lambda_r) < 0~~\textrm{for}~i \in \{r+1, \dots, m\}.$$ Note that we need $E_r^\top \overline{U}(0)$ to be invertible to derive the upper bound \eqref{paupperbound}. Under this condition, $\tilde{\gamma}^2_i(t)$ converges to zero. However, when $E_r^\top \overline{U}(0)$ is not invertible, the algorithm starts at a saddle point, and \eqref{ODE} no longer applies. As can be seen, the ODE characterization is insufficient to capture the local dynamics (e.g., around saddle points or global optima) of the algorithm. 

\subsection{Local Dynamics by SDE}\label{gdsde}

The deterministic ODE characterizes the average behavior of the solution trajectory. To capture the uncertainty of the local algorithmic behavior, we need to rescale the influence of the noise to bring the randomness back, which leads us to a stochastic differential equation (SDE) approximation.

\subsubsection{Stage 1: Escape from Saddle Points}

Recall that $\Lambda = \textrm{diag}(\lambda_1, \dots, \lambda_m)$ collects all the eigenvalues of $\Sigma$. We consider the following eigenvalue decomposition $$\overline{U}^\top(0) \Lambda \overline{U}(0) = Q^\top \tilde{\Lambda} Q,$$ where $Q \in \mathbb{R}^{r \times r}$ is orthogonal and $\tilde{\Lambda} = \textrm{diag}(\tilde{\lambda}_1, \dots, \tilde{\lambda}_r)$. Again, by a continuous time interpolation, we denote $$\zeta_{ij, \eta}(t) = \eta^{-1/2} e'^\top_j Q \left(\overline{U}^\eta(t)\right)^\top e_i,$$ where $e'_j$ is the canonical basis in $\mathbb{R}^r$. Then we decompose the principle angle $\gamma^2_{i, \eta}(t)$ as $$\gamma_{i, \eta}^2(t) = \eta \sum_{j=1}^r \zeta^2_{ij, \eta}(t).$$ Recall that $\overline{U}(0)$ is a saddle point, if the column span of $\overline{U}(0)$ equals the subspace spanned by $\{e_{a_1}, \dots, e_{a_r}\}$ with $\cA_r = \{a_1, \dots, a_r\} \neq \{1, \dots, r\}$. Therefore, if the algorithm starts around a saddle point, there exists at least one $i \in \{1, \dots, r\}$ such that $$\gamma_{i, \eta}^2(0) \approx 0~~\textrm{and}~~\gamma_{a, \eta}^2(0) \approx 1$$ for $a \in \cA_r$. The next theorem captures the uncertainty of $\gamma_{i, \eta}^2(t)$ around a saddle point.
\begin{theorem}\label{stage1sde}
Suppose $\overline{U}(0)$ is initialized around a saddle point corresponding to $\cA_r$. As $\eta \rightarrow 0$, conditioning on the event $$\gamma_{i, \eta}^2(t) = O(\eta)~~\textrm{for~some}~i \in \{1, \dots, r\},$$ $\zeta_{ij, \eta}(t)$ weakly converges to the solution of the following stochastic differential equation
\begin{align}\label{SDE}
d\zeta_{ij} = K_{ij} \zeta_{ij} dt + G_{ij} dB_t,
\end{align}
where $B_t$ is a standard Brownian motion. We have $$K_{ij} \in \left [\lambda_i - \lambda_1, \lambda_i - \lambda_{a_r} \right]$$ with $a_r$ being the largest element in $\cA_r$, and $G_{ij}^2 < \infty$.
\end{theorem}
The proof of Theorem \ref{stage1sde} is provided in Appendix \ref{sdeproof}. Here we use the Infinitesimal Generator Approach (IGA) again. Specifically, given the increment, $$\Delta \zeta_{ij, \eta}(t) = \zeta_{ij, \eta}(t+\eta) - \zeta_{ij,\eta}(t),$$ we show that the infinitesimal conditional mean and variance satisfy similar conditions in \eqref{odemean} and \eqref{odevariance} but with bounded variance. We remark that the event $\gamma_{i, \eta}^2(t) = O(\eta)$ is only a technical assumption. This does not cause any issue, since when $\eta^{-1} \gamma_{i, \eta}^2(t)$ is large, the algorithm has already escaped from the saddle point.

Note that \eqref{SDE} admits the analytical solution
\begin{align}\label{sdesolution}
\zeta_{ij}(t) = \zeta_{ij}(0) e^{K_{ij}t} + G_{ij} \int_0^t e^{-K_{ij}(s-t)} dB(s),
\end{align}
which is known as the Ornstein-Uhlenbeck (O-U) process. The uncertainty of $\zeta_{ij}(t)$ is precisely characterized by the stochastic integral part. We give the following implications based on different values of $K_{ij}$:

\noindent {\bf (a)}. When $K_{ij} > 0$, rewrite \eqref{sdesolution} as $$\zeta_{ij}(t) = \underbrace{\left [\zeta_{ij}(0) + G_{ij} \int_0^t e^{-K_{ij}s} dB(s) \right ]}_{T_1} \underbrace{e^{K_{ij}t}}_{T_2}.$$ The exponential term $T_2$ is dominant and increases to positive infinity as $t$ goes to infinity. While $T_1$ is a process with mean $\zeta_{ij}(0)$ and variance bounded by $\frac{G_{ij}^2}{2K_{ij}}$. Hence, $T_2$ acts as a driving force to increase $\zeta_{ij}(t)$ exponentially fast so that $\zeta_{ij}(t)$ quickly gets away from 0;

\noindent {\bf (b)}. When $K_{ij} < 0$, the mean of $\zeta_{ij}(t)$ is $\zeta_{ij}(0) e^{K_{ij}t}$. The initial condition restricts $\zeta_{ij}(0)$ to be small. Thus as $t$ increases, the mean of $\zeta_{ij}(t)$ converges to zero. This implies that the drift term vanishes quickly. The variance of $\zeta_{ij}(t)$ is bounded by $\frac{G_{ij}^2}{-2K_{ij}}$. Hence, $\zeta_{ij}(t)$ roughly oscillates around 0;

\noindent {\bf (c)}. When $K_{ij} = 0$, the drift term is approximately zero, which implies that $\zeta_{ij}(t)$ also oscillates around 0.

We provide an example showing how the algorithm escapes from a saddle point. Suppose that the algorithm starts at the saddle point with approximately the following principle angle loading,
\begin{align*}
1, \dots, 1, \underbrace{0}_{q\textrm{-th}~\textrm{position}}, 1, \dots, 1, \underbrace{1}_{p\textrm{-th}~\textrm{position}}, 0, \dots, 0.
\end{align*}
Consider the principle angle $\gamma_{q, \eta}^2(t)$. By implication (a), we know $$K_{qr} = \lambda_q - \lambda_p > 0.$$ Hence $\zeta_{qr, \eta}(t)$ increases quickly away from zero. Thus, $$\gamma_{q, \eta}^2(t) \geq \eta \zeta^2_{qr, \eta}(t)$$ also increases quickly, which drives the algorithm away from the saddle point. Meanwhile, by (b) and (c), $\gamma_{i, \eta}^2(t)$ stays at 1 for $i > q$ because of the vanishing drift. The algorithm tends to escape from the saddle point through reducing $\gamma_{p,\eta}^2(t)$, since this yields the largest eigengap, $\lambda_q - \lambda_p$. When we have $$q = r~~\textrm{and}~~p = r+1,$$ the eigengap is minimal. Thus, it is the worst situation for the algorithm to escape from a saddle point. Then we have the following proposition. 

\begin{proposition}\label{saddletime}
Suppose that the algorithm starts in the vicinity of the saddle point corresponds to $\cA_r$, where $$\cA_r = \{1, \dots, r-1, r+1\}.$$
Given a pre-specified $\nu$ and $\delta = O(\eta^{\frac{1}{2}})$ for a sufficiently small $\eta$, we need
\begin{align}
T_1 \asymp \frac{1}{\lambda_r - \lambda_{r+1}} \log(K + 1) \notag
\end{align}
such that $\PP\left(\gamma_{r, \eta}^2(T_1) \geq \delta^2\right) \geq 1 - \nu,$ where $K = \frac{2(\lambda_r - \lambda_{r+1})\eta^{-1}\delta^2}{\left [\Phi^{-1}(\frac{1-\nu/2}{2}) \right ]^2 G_{rr}^2}$, and $\Phi$ is the CDF of the standard Gaussian distribution.
\end{proposition}

The proof of Proposition \ref{saddletime} is provided in Appendix \ref{saddletime-proof}. Proposition \ref{saddletime} implies that, in an asymptotic sense, we need
\begin{align}
S_1 \asymp \frac{T_1}{\eta} \asymp \frac{1}{\eta(\lambda_r - \lambda_{r+1})} \log \left(\frac{2(\lambda_r - \lambda_{r+1})\eta^{-1}\delta^2}{\left [\Phi^{-1}(\frac{1-\nu/2}{2}) \right ]^2 G_{rr}^2} + 1\right) \notag
\end{align}
iterations to escape from a saddle point and the algorithm enters the second stage.

\subsubsection{Stage 2: Traverse between Stationary Points}

After the algorithm escapes from the saddle point, the gradient is dominant, and the influence of noise is negligible. Thus, the algorithm behaves like an almost deterministic traverse between stationary points, which can be viewed as a two-step discretization of the ODE with an error of the order $O(\eta)$ \citep{griffiths2010numerical}. Hence, we focus on the principle angle $\gamma_{i, \eta}^2(t)$ to characterize the traverse of the algorithm in this stage. Recall that we assume $\cA_r = \{1, \dots, r-1, r+1\}.$ When the algorithm escapes from the saddle point, we have $\gamma_{r, \eta}^2(T_1) \geq \delta^2$, which implies $$\sum_{i=r+1}^m \gamma_{i, \eta}^2(t) \leq 1 - \delta^2.$$ The following proposition assumes that the algorithm starts at this initial condition.
\begin{proposition}\label{stage2time}
After restarting the counter of time, for a sufficiently small $\eta$ and $\delta = O(\eta^{\frac{1}{2}})$. We need
\begin{align}
T_2 \asymp \frac{1}{\lambda_r - \lambda_{r+1}} \log \frac{1}{\delta^2} \notag
\end{align} 
such that $\PP\left(\sum_{i=r+1}^m \gamma_{i, \eta}^2(T_2) \leq \delta^2\right) \geq \frac{3}{4}$.
\end{proposition}
The proof of Proposition \ref{stage2time} is provided in Appendix \ref{stage2time-proof}. Proposition \ref{stage2time} further implies that, in an asymptotic sense, we need
\begin{align}
S_2 \asymp \frac{T_2}{\eta} \asymp \frac{1}{\eta(\lambda_r - \lambda_{r+1})} \log \frac{1}{\delta^2} \notag
\end{align}
iterations to reach the neighborhood of the global optima.

\subsubsection{Stage 3: Converge to Global Optima}

Again, we restart the counter of time. The following theorem characterizes the dynamics of the algorithm around the global optima. Similar to stage 1, we rescale the noise to quantify the uncertainty of the algorithm by an SDE.
Hence, we decompose the principle angle as $$\gamma_{i, \eta}^2(t) = \eta \sum_{j=1}^r \zeta^2_{ij, \eta}(t).$$
\begin{theorem}\label{stage3sde-thm}
Suppose $\overline{U}(0)$ is initialized around the global optima with $\sum_{i=r+1}^m \gamma_{i, \eta}^2(0) = O(\eta)$. Then as $\eta \rightarrow 0$, for $i = r+1, \dots, m$ and $j = 1, \dots, r$, $\zeta_{ij, \eta}(t)$ weakly converges to the solution of the following SDE
\begin{align}\label{stage3sde}
d\zeta_{ij} = K_{ij} \zeta_{ij} dt + G_{ij} dB_t,
\end{align}
where $B_t$ is a standard Brownian motion, $$K_{ij} \in [\lambda_i - \lambda_1, \lambda_i - \lambda_r],~~\textrm{and}~~G_{ij}^2 < \infty.$$
\end{theorem}
The analytical solution of \eqref{stage3sde} is
\begin{align*}
\zeta_{ij}(t) = \zeta_{ij}(0) e^{K_{ij}t} + G_{ij} \int_0^t e^{-K_{ij}(s-t)} dB(s)
\end{align*}
Note that we have $K_{ij} \leq \lambda_{r+1} - \lambda_r < 0$. We remark here again that the uncertainty of the algorithm around the global optima is precisely characterized by the stochastic integral. The mean and variance of $\zeta_{ij}(t)$ satisfy
\begin{align}
& \EE[\zeta_{ij}(t)] = \zeta_{ij}(0) e^{K_{i, j}t} \rightarrow 0, ~\textrm{as}~~ t \rightarrow \infty, ~~\textrm{and}~~ \notag \\
& \textrm{Var}(\zeta_{ij}(t)) = \frac{G_{ij}^2}{-2K_{ij}}(1 - e^{2K_{ij}t}) \leq \frac{G_{ij}^2}{-2K_{ij}}. \notag
\end{align}
The proof of Theorem \ref{stage3sde-thm} is provided in Appendix \ref{stage3sde-thm-proof}. We further establish the following proposition.
\begin{proposition}\label{stage3time}
For sufficiently small $\epsilon > 0$ and $\eta$, $\delta = O(\eta^{\frac{1}{2}})$, given $\sum_{i=r+1}^m \gamma_{i, \eta}^2(0) \leq \delta^2$, after restarting the counter of time, we need
\begin{align}
T_3 \asymp \frac{1}{\lambda_r - \lambda_{r+1}} \log K' \notag
\end{align}
such that $\PP\left(\sum_{i=r+1}^m \gamma_{i, \eta}^2(T_3) \leq \epsilon\right) \geq \frac{3}{4},$ where $K' = \frac{8(\lambda_r - \lambda_{r+1})\delta^2}{(\lambda_r - \lambda_{r+1})\epsilon - 4 \eta r G_m}$ and $G_m = \max_{1 \leq j \leq r} \sum_{i=r+1}^m G_{ij}^2$. 
\end{proposition}
The subscript $m$ in $G_m$ highlights its dependence on the dimension $m$. The proof of Proposition \ref{stage3time} is provided in Appendix \ref{stage3time-proof}. Proposition \ref{stage3time} implies that, in an asymptotic sense, we need
\begin{align}
S_3 \asymp \frac{T_3}{\eta} \asymp \frac{1}{\eta(\lambda_r - \lambda_{r+1})} \log \frac{8(\lambda_r - \lambda_{r+1})\delta^2}{(\lambda_r - \lambda_{r+1})\epsilon - 4 \eta r G_m} \notag
\end{align}
iterations to converge to an $\epsilon$-optimal solution in the third stage. Combining all the results in the first two stages, we know that, in an asymptotic sense, after $T_1 + T_2 + T_3$ time, the algorithm converges to an $\epsilon$-optimal solution asymptotically. This further leads us to a more refined result in the following corollary.
\begin{corollary}{\label{totalcor}}
For a sufficiently small $\epsilon$, we choose
\begin{align}
& \eta \asymp \frac{(\lambda_r - \lambda_{r+1})\epsilon}{5rG_m}. \notag
\end{align}
We need $$T \asymp \frac{1}{\lambda_r - \lambda_{r+1}} \log \frac{rG_m}{\epsilon(\lambda_r - \lambda_{r+1})}$$
time such that $$\PP\left(\left \lVert \cos \Theta\left(\overline{E}_r, \overline{U}^\eta (T)\right) \right \rVert_{\textrm{F}}^2 \leq \epsilon\right) \geq \frac{3}{4}.$$
\end{corollary}
The proof of Corollary \ref{totalcor} is provided in Appendix \ref{totalcor-proof}. Corollary \ref{totalcor} further implies that, in an asymptotic sense, after
\begin{align}
S \asymp \frac{T}{\eta} \asymp \frac{rG_m}{\epsilon(\lambda_r - \lambda_{r+1})^2} \log \frac{rG_m}{\epsilon(\lambda_r - \lambda_{r+1})} \notag
\end{align}
iterations, we achieve an $\epsilon$-optimal solution. Recall that we choose the block size $h$ of downsampling to be $h = O\left (\kappa_\rho \log \frac{1}{\eta} \right )$. Thus, the asymptotic sample complexity satisfies $$N \asymp \frac{rG_m}{\epsilon(\lambda_r - \lambda_{r+1})^2} \log^2 \frac{rG_m}{\epsilon(\lambda_r - \lambda_{r+1})}.$$
From the perspective of statistical recovery, the obtained estimator $\hat{U}$ enjoys a near optimal asymptotic rate of convergence $$\left\lVert \cos \Theta(\hat{U}, U^*)\right\rVert^2_{\textrm{F}} \asymp \frac{rG_m\log N}{(\lambda_r - \lambda_{r+1})^2N/\kappa_\rho},$$ where $N$ is the number of data points.

\subsection{Extension to Generalized Hebbian Algorithm}\label{connection}

We connect Oja's algorithm to stochastic generalized Hebbian algorithm (GHA), which updates as follows,
\begin{align}
U_{s+1} = U_s + \eta (I - U_s U_s^\top) X_s U_s. \notag
\end{align}
Note that $(I - U_s U_s^\top) X_s U_s$ is the gradient on the Stiefel manifold, when $U_s^\top U_s = I_r$. As can be seen, GHA is essentially the first order approximation of Oja's algorithm without the remainder. As $\eta \rightarrow 0$, the remainder vanishes. Thus, GHA and Oja's algorithm share the same diffusion approximations for streaming PCA problems.


\section{Numerical Experiments}

We demonstrate the effectiveness of our proposed algorithm using both simulated and real datasets.

\subsection{Simulated Data}

We first verify our analysis of streaming PCA problems for time series using a simulated dataset. We choose a Gaussian VAR model with dimension $m = 16$. The random vector $\epsilon_k$'s are independently sampled from $N(0, S)$, where 
\begin{align*}
S = \textrm{diag}(&1, 1,  1, 1, 1, 1, 1, 1, 1, 1, 1, 1, 1, 3, 3, 3).
\end{align*}
We choose the coefficient matrix $A = V^\top D V$, where $V \in \mathbb{R}^{16 \times 16}$ is an orthogonal matrix that we randomly generate, and $D = 0.1 D_0$ is a diagonal matrix satisfying
\begin{align*}
D_0 = \textrm{diag}(0.68, 0.68, 0.69, 0.70, 0.70, 0.70, 0.72, 0.72, 0.72, 0.72, 0.72, 0.72, 0.80, 0.80, 0.85, 0.90).
\end{align*}
By solving the discrete Lyapunov equation $\Sigma = A\Sigma A^\top + S$, we calculate the covariance matrix of the stationary distribution, which satisfies $\Sigma = U^\top \Lambda U$, where $U \in \mathbb{R}^{16 \times 16}$ is orthogonal and
\begin{align*}
\Lambda = \textrm{diag}(&3.0175, 3.0170, 3.0160, 1.0077, 1.0070, 1.0061, 1.0058, 1.0052, \\
&1.0052, 1.0052, 1.0052, 1.0051, 1.0049, 1.0049, 1.0047, 1.0047).
\end{align*}
We aim to find the leading principle components of $\Sigma$ corresponding to the first 3 largest eigenvalues. Thus, the eigengap is $\lambda_3 - \lambda_4 = 2.0083$. We initialize the solution at the saddle point whose column span is the subspace spanned by the eigenvectors corresponding to 3.0175, 3.0170 and 1.0070. The step size is $\eta = 3 \times 10^{-5}$, and the algorithm runs with $8 \times 10^5$ total samples. The trajectories of the principle angle over 20 independent simulations with block size $h = 4$ are shown in Figure \ref{3stages}. We can clearly distinguish three different stages. Figure \ref{saddledist} and \ref{optimadist} illustrate that entries of principle angles, $\zeta_{33}$ in stage 1 and $\zeta_{42}$ in stage 3, are Ornstein-Uhlenbeck processes. Specifically, the estimated distributions of $\zeta_{33}$ and $\zeta_{42}$ over 100 simulations follow Gaussian distributions. We can check that the variance of $\zeta_{33}$ increases in stage 1 as iteration increases, while the variance of $\zeta_{42}$ in stage 3 approaches a fixed value. All these simulated results are consistent with our theoretical analysis.
\begin{figure}[htb!]
\centering
\begin{subfigure}[t]{0.5\linewidth}
\centering
\includegraphics[width = 0.8\textwidth]{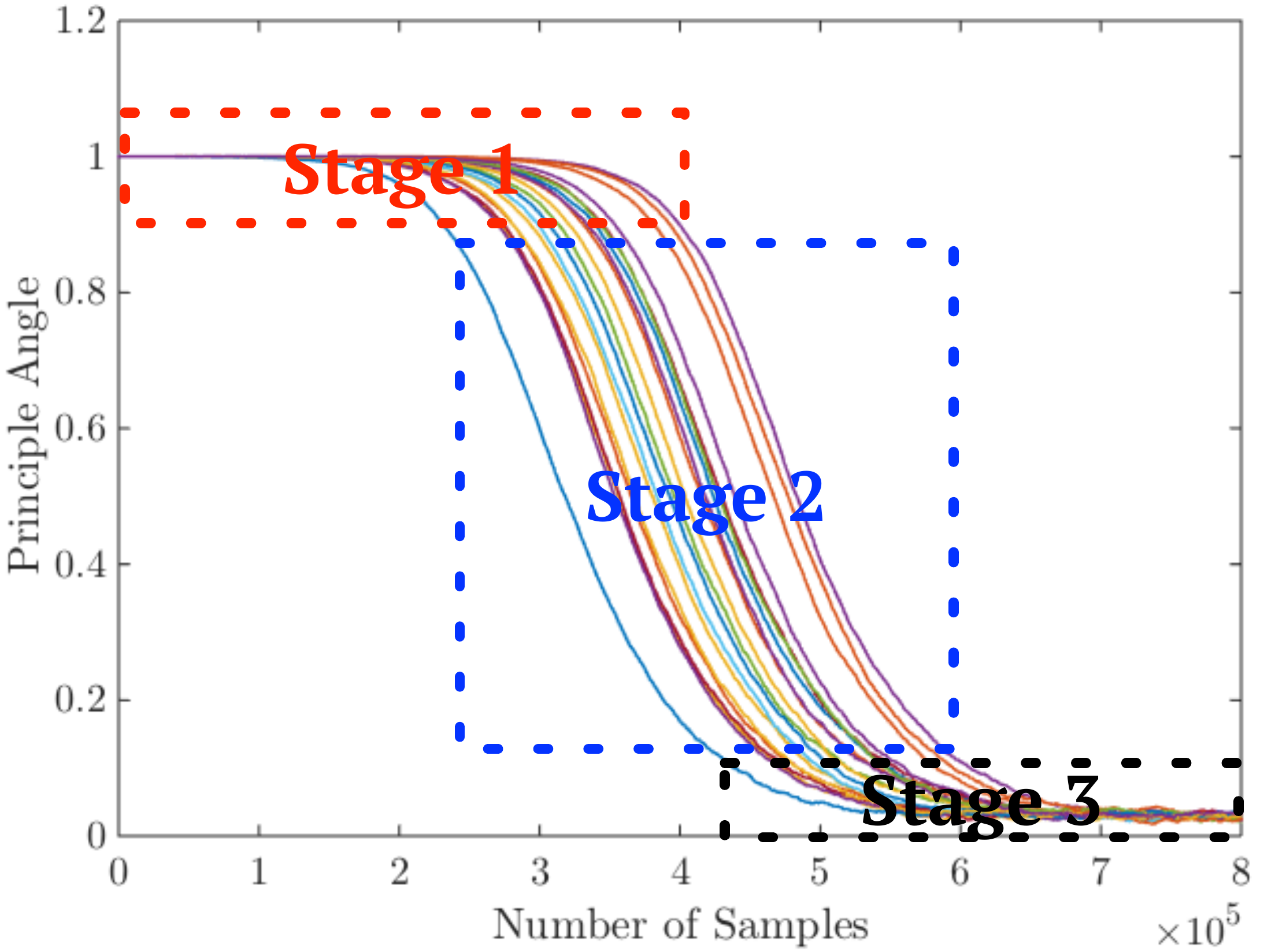}
\caption{Solution trajectories}
\label{3stages}
\end{subfigure}%
\begin{subfigure}[t]{0.5\linewidth}
\centering
\includegraphics[width = 0.8\textwidth]{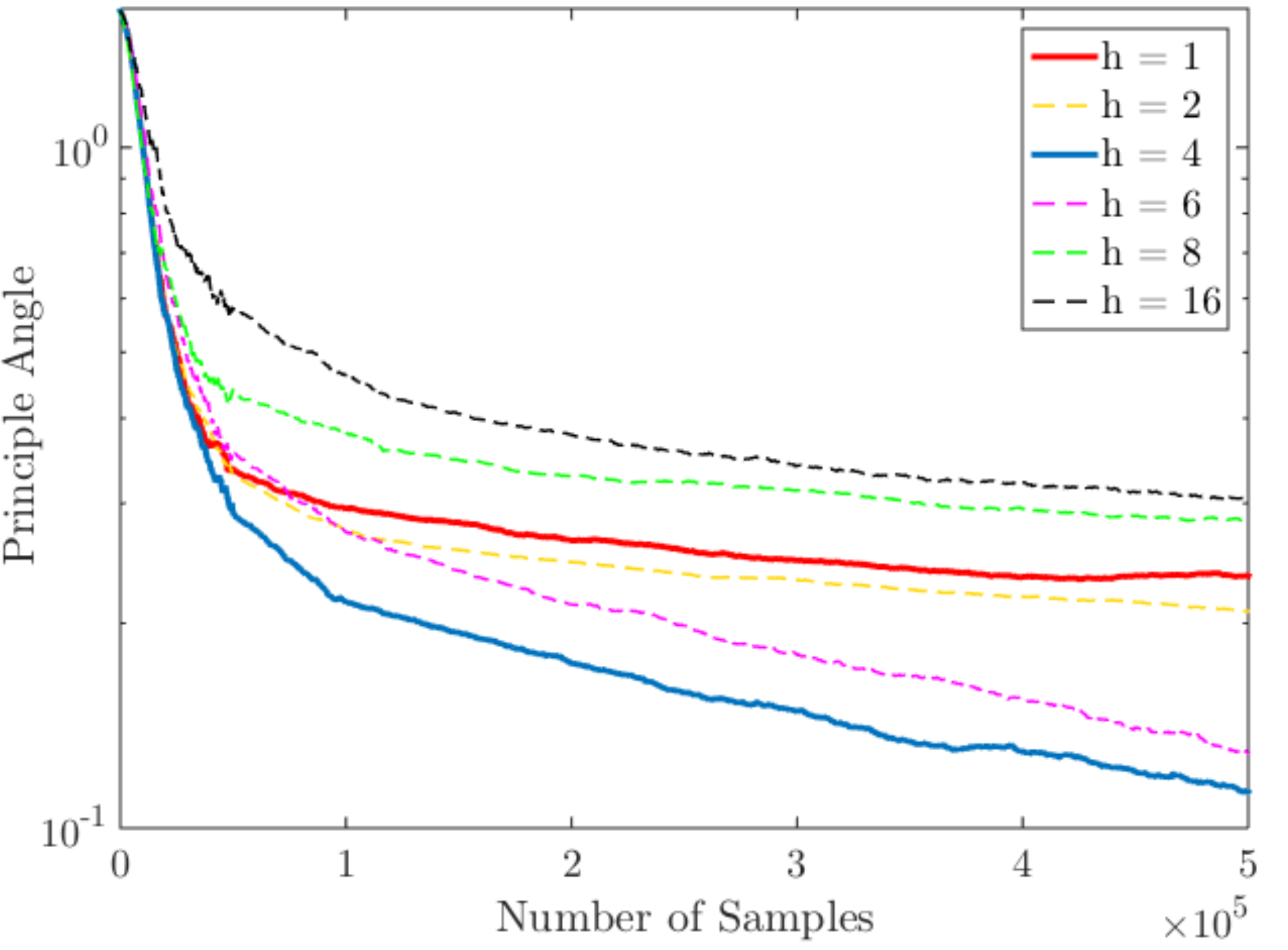}
\caption{Different block sizes}
\label{bestblock}
\end{subfigure}%

\begin{subfigure}[t]{0.5\linewidth}
\centering
\includegraphics[width = 0.8\textwidth]{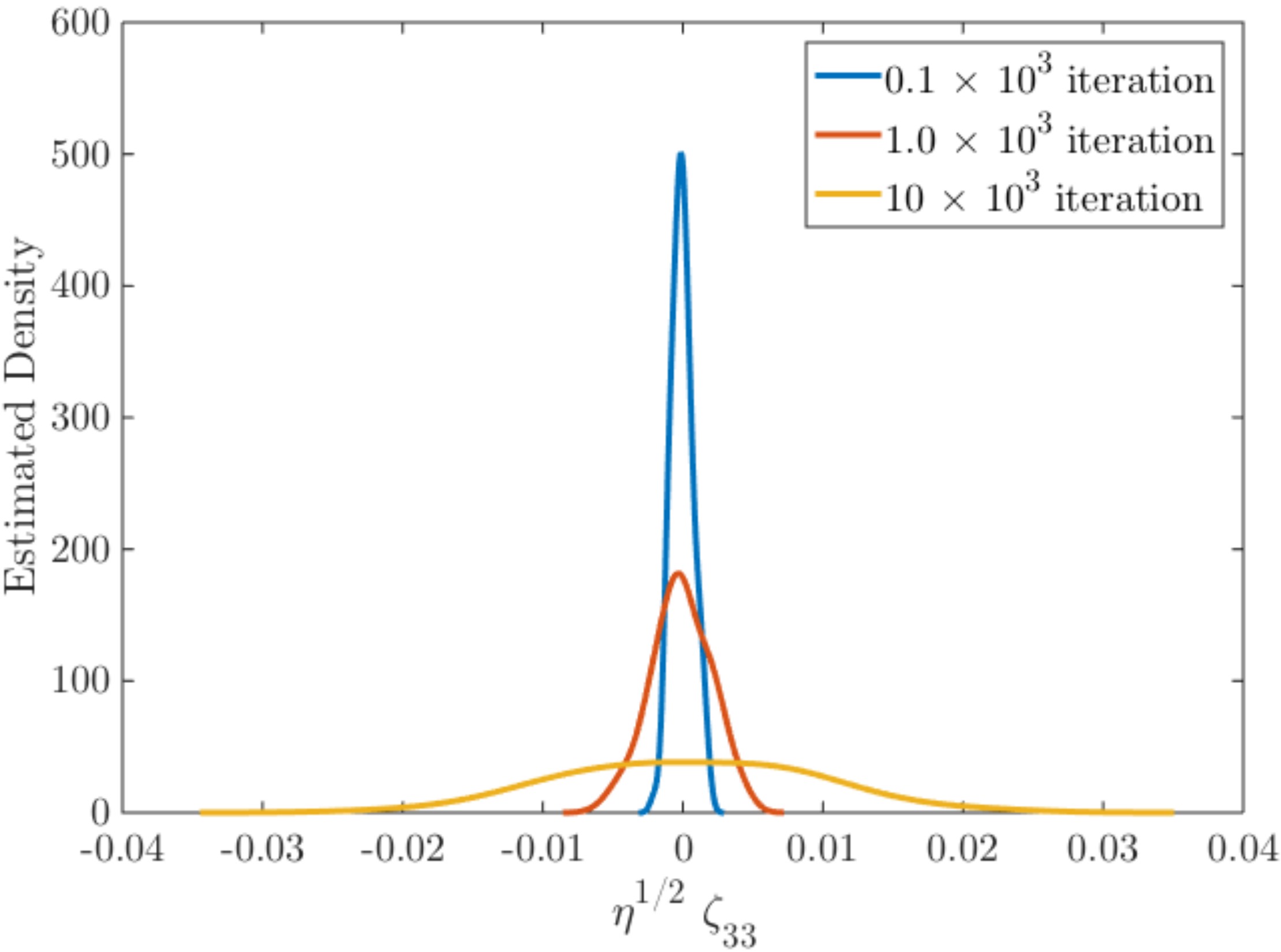}
\caption{Distribution of $\zeta_{33}(t)$}
\label{saddledist}
\end{subfigure}%
\begin{subfigure}[t]{0.5\linewidth}
\centering
\includegraphics[width = 0.8\textwidth]{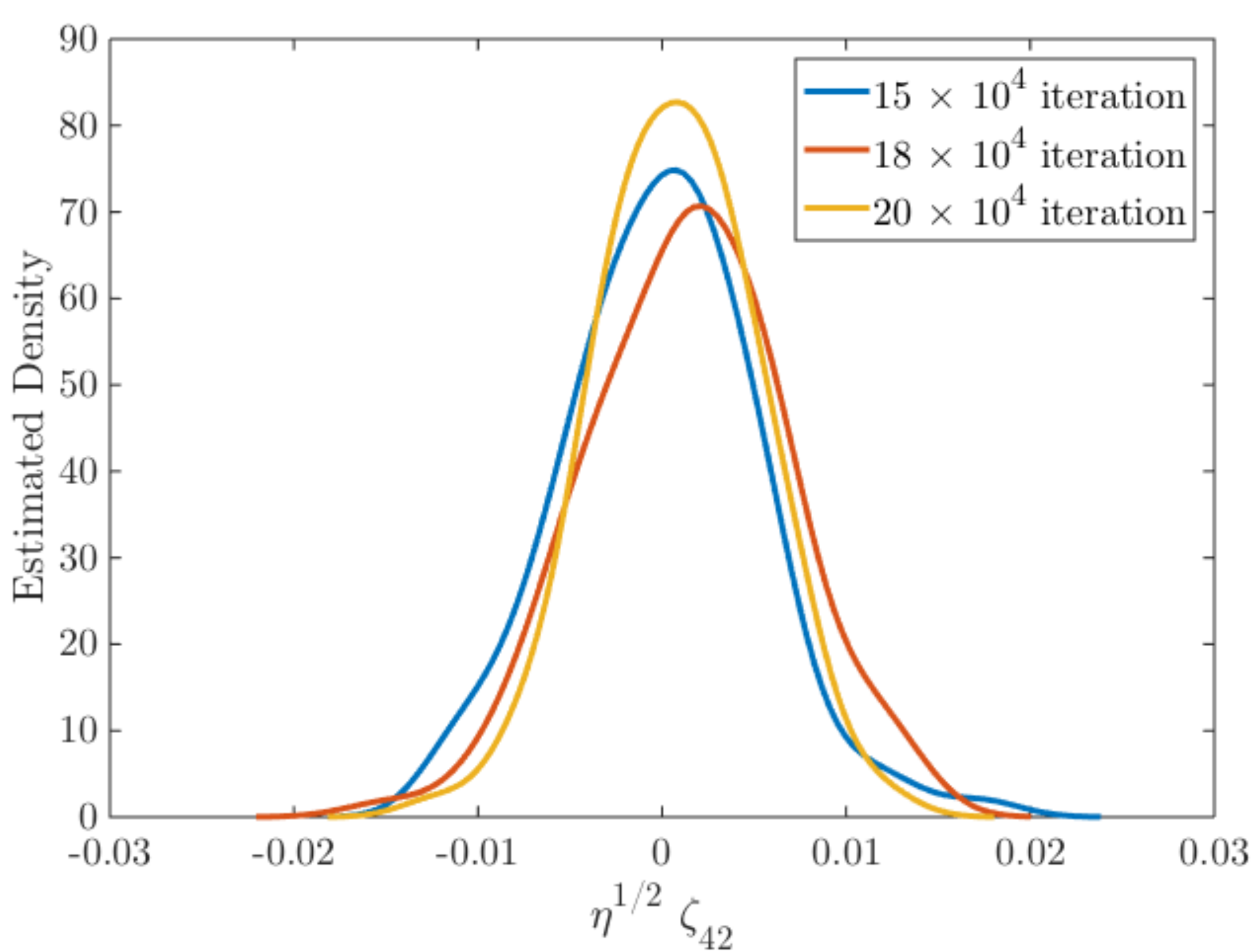}
\caption{Distribution of $\zeta_{42}(t)$}
\label{optimadist}
\end{subfigure}
\caption{Illustrations of various algorithmic behaviors in simulated examples:
(a) presents three stages of the algorithm; (b) compares the performance of different block sizes; (c) and (d) demonstrate the Ornstein-Uhlenbeck processes of $\zeta_{33}$ in stage 1 and $\zeta_{42}$ in stage 3.
}
\label{simulated}
\end{figure}

We further compare the performance of different block sizes of downsampling with step size annealing. We keep using Gaussian VAR model with $D = 0.9 D_0$ and
\begin{align*}
S = \textrm{diag}(1.45, 1.45,  1.45, 1.45, 1.45, 1.45, 1.45, 1.45, 1.45, 1.45, 1.45, 1.45, 1.45, 1.455, 1.455, 1.455).
\end{align*}
The eigengap is $\lambda_3 - \lambda_4 = 0.0025$. We run the algorithm with $5 \times 10^5$ samples and the chosen step sizes vary according to the number of samples $k$. Specifically, we set the step size $\eta = \eta_0 \times \frac{h}{4000}$ if $k < 2 \times 10^4$, $\eta = \eta_0 \times \frac{h}{8000}$ if $k \in [2 \times 10^4, 5 \times 10^4)$, $\eta = \eta_0 \times \frac{h}{48000}$ if $k \in [5 \times 10^4, 10 \times 10^4)$, and $\eta = \eta_0 \times \frac{h}{120000}$ if $k \geq 10 \times 10^4$. We choose $\eta_0$ in $\{0.125, 0.25, 0.5, 1, 2\}$ and report the final principle angles achieved by different block sizes $h$ in Table \ref{besteta}. Figure \ref{bestblock} presents the averaged principle angle over $5$ simulations with $\eta_0 = 0.5$. As can be seen, choosing $h=4$ yields the best performance. Specifically, the performance becomes better as $h$ increases from 1 to around 4. However, the performance becomes worse, when $h = 16$ because of the lack of iterations.

\begin{table}[htb!]
\centering
\begin{tabular}{| c | c | c | c | c | c |}
\hline
& $\eta_0 = 0.125$ & $\eta_0 = 0.25$ & $\eta_0 = 0.5$ & $\eta_0 = 1$ & $\eta_0 = 2$ \\\hline
$h = 1$ & 0.7775 & 0.3595 & \textbf{0.2320} & 0.2449 & 0.3773 \\\hline
$h = 2$  & 0.7792 & 0.3569 & \textbf{0.2080} & 0.2477 & 0.2290 \\\hline
$h = 4$  & 0.7892 & 0.3745 & \textbf{0.1130} & 0.3513 & 0.4730 \\\hline
$h = 6$  & 0.7542 & 0.3655 & \textbf{0.1287} & 0.3317 & 0.3983 \\\hline
$h = 8$  & 0.7982 & 0.3933 & \textbf{0.2828} & 0.3820 & 0.4102  \\\hline 
$h = 16$ & 0.7783  & 0.4324 & \textbf{0.3038} & 0.5647 & 0.6526 \\\hline
\end{tabular}
\caption{The final principle angles achieved by different block sizes with varying $\eta_0$.}
\label{besteta}
\end{table}

\subsection{Real Data}

We adopt the Air Quality dataset \citep{de2008field}, which contains 9358 instances of hourly averaged concentrations of totally 9 different gases in a heavily polluted area. We remove measurements with missing data and then normalize all the data points by subtracting their sample mean and dividing by their sample standard deviation. We aim to estimate the first 2 principle components of the series. We randomly initialize the algorithm, and choose the block size of downsampling to be 1, 3, 5, 10, and 60. Figure \ref{airquality} shows that the projection of each data point onto the leading and the second principle components. We also present the results of projecting data points onto the eigenspace of sample covariance matrix indicated by Batch in Figure \ref{airquality}. All the projections have been rotated such that the leading principle component is parallel to the horizontal axis. As can be seen, when $h=1$, the projection yields some distortion in the circled area. When $h=3$ and $h=5$, the projection results are quite similar to the Batch result. As $h$ increases, however, the projection displays obvious distortion again compared to the Batch result. The concentrations of gases are naturally time dependent. Thus, we deduce that the distortion for $h=1$ comes from the data dependency, while for the case $h=60$, the distortion comes from the lack of updates. This phenomenon coincides with our simulated data experiments.
\begin{figure}[htb!]
\centering
\includegraphics[width = 0.75\textwidth]{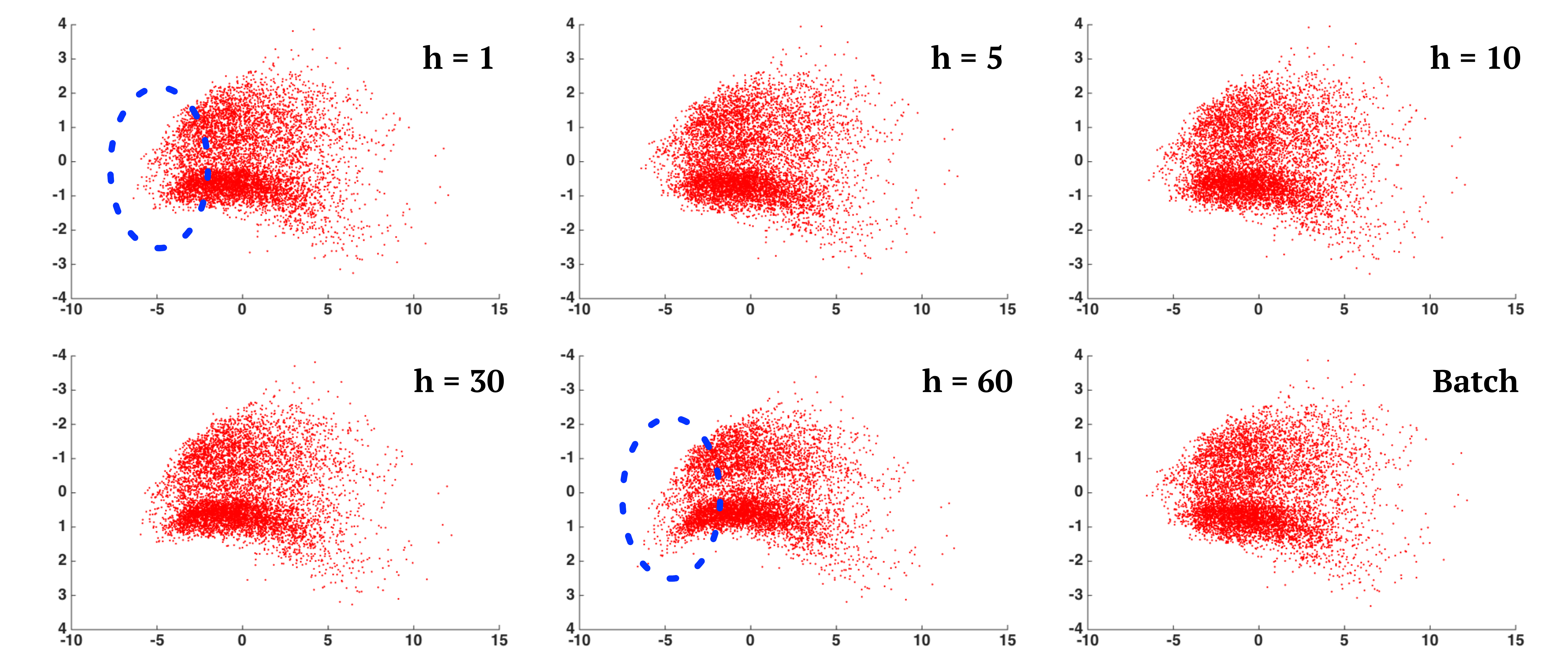}
\caption{Projections of air quality data onto the leading and the second principle components with different block sizes of downsampling. 
We highlight the distortions for $h=1$ and $h=60$.
}
\label{airquality}
\end{figure}



\section{Discussions}\label{discussion}

Our analysis requires $\Sigma$ to be positive definite in Assumption \ref{assump1}, which actually can be relaxed. Specifically, we can inject a small perturbation $\nu_k$ to $z_k$ at each iteration, where $\nu_k$'s are independently sampled from $N(0,\epsilon I)$, and $\epsilon$ is our prespecified optimization error. Then we are essentially recovering the span of the leading eigenvectors of the covariance matrix $\tilde{\Sigma} = \Sigma + \epsilon I$, which is identical to that of $\Sigma$.

We remark that our analysis characterizes how our proposed algorithm escapes from the saddle point. This is not analyzed in the related work, \cite{allen2016first}, since they use random initialization. Note that our analysis also applies to random initialization, and directly starts with the second stage.

Our analysis is inspired by diffusion approximations in existing applied probability literature \citep{glynn1990diffusion, freidlin1998random, kushner2003stochastic, ethier2009markov}, which target to capture the uncertainty of stochastic algorithms for general optimization problems. Without explicitly specifying the problem structures, these analyses usually cannot lead to concrete convergence guarantees. In contrast, we dig into the optimization landscape of the streaming PCA problem. This eventually allows us to precisely characterize the algorithmic dynamics and provide concrete convergence guarantees, which further lead to a deeper understanding of the uncertainty in nonconvex stochastic optimization.

We believe the following directions should be of interest:

\noindent $\bullet$ Our results are asymptotic. We need more analytical tools to bridge the asymptotic results to the algorithm. How to connect our analysis to nonasymptotic results should be an important direction.

\noindent $\bullet$ Our results consider a fixed step size $\eta \rightarrow 0$. However, the step size annealing yields good empirical performance. Thus, how to generalize the analysis to the step size annealing is another important direction.

\noindent $\bullet$ Our results are based on the geometric ergodicity assumption. How to weaken this assumption and generalize the analysis to a larger class of models for dependent data is a challenging but interesting future direction.


\bibliographystyle{ims}
\bibliography{ref}


\newpage
\onecolumn
\appendix
\section{Detailed Proofs in Section \ref{downsample}}
\subsection{Proof of Lemma \ref{unbiaslemma}}
\begin{proof}
We first assume the stationary distribution has zero mean and denote the covariance matrix as $\Sigma$. The total variation distance of $p^h(z, \cdot)$ and $\pi(\cdot)$ is equivalent to
\begin{align*}
\cD_{\textrm{TV}}(p^h(z, \cdot), \pi(\cdot)) = \frac{1}{2} \int \left |p^h(z, x) - \pi(x) \right | dx.
\end{align*}
Then we try to find the conditional expectation,
\begin{align*}
\E \left [z_{k+h}z_{k+h}^\top | z_k \right ] & = \int xx^\top p^h(z_k, x) dx \\
& = \int xx^\top (\pi(x) + p^h(z_k, x) - \pi(x)) dx \notag \\
& = \int xx^\top \pi(x) dx + \int xx^\top (p^h(z_k, x) - \pi(x)) dx \notag \\
& = \Sigma + \int xx^\top (p^h(z_k, x) - \pi(x)) dx. \notag
\end{align*}
We bound the second term by the following,
\begin{align*}
\left \lVert \int xx^\top (p^h(z_k, x) - \pi(x)) dx \right \rVert_2 & \leq \int \lVert x \rVert^2_2 \left |p^h(z_k, x) - \pi(x) \right | dx \\
& \leq \int_{\lVert x \rVert^2_2 \leq t} \lVert x \rVert^2_2 \left |p^h(z_k, x) - \pi(x) \right | dx + \int_{\lVert x \rVert_2 > t} \lVert x \rVert^2_2 \left |p^h(z_k, x) - \pi(x) \right | dx \notag \\
& \leq C_1 t \rho^h + \int_{\lVert x \rVert^2_2 > t} \lVert x \rVert^2_2 \left |p^h(z_k, x) - \pi(x) \right | dx \notag \\
& \leq C_1 t \rho^h + \int_{\lVert x \rVert^2_2 > t} \lVert x \rVert^2_2 p^h(z_k, x) dx + \int_{\lVert x \rVert^2_2 > t} \lVert x \rVert_2 \pi(x) dx \notag \\
& \leq C_1 t \rho^h + \int_t^\infty \mathbb{P}_{p^h} (\lVert x \rVert^2_2 > s) ds + \int_t^\infty \mathbb{P}_{\pi} (\lVert x \rVert^2_2 > s) ds. \notag
\end{align*}
By our assumption, $x$ is a Sub-Gaussian random vector, then $\mathbb{P}_{p^h} (\lVert x \rVert_2 > t) \leq C_2 \exp(-C_3 t^2)$ and $\mathbb{P}_{\pi} (\lVert x \rVert_2 > t) \leq C'_2 \exp(-C'_3 t^2)$. The integration is bounded by
\begin{align*}
\int_{\sqrt{t}}^\infty \exp(-s^2) ds = \int_0^\infty \exp(-(s + \sqrt{t})^2) ds \leq \exp(-t) \int_0^\infty \exp(-2s\sqrt{t}) ds = \frac{1}{\sqrt{t}} \exp(-t).
\end{align*}
Thus, we have $\left \lVert \int xx^\top (p^h(z_k, x) - \pi(x)) dx \right \rVert_2 \leq C_1 t \rho^h + C_2 \frac{1}{\sqrt{t}} e^{-C_3 t}$. Optimize over $t$ and neglect the exponential term, we pick $t = O\left (\rho^{-\frac{2h}{3}} \right)$ to reach $\left \lVert \int xx^\top (p^h(z_k, x) - \pi(x)) dx \right \rVert_2 \leq O(\rho^{h/3})$. Therefore, we have $\E \left [z_{k+h}z_{k+h}^\top | z_k \right ] = \Sigma + E \Sigma$ with $\lVert E \rVert_2 = O(\rho^{h/3})$, which implies that if we pick $h = O\left(\frac{1}{1-\rho} \log \frac{1}{\tau} \right)$, then we have $\lVert E \rVert_2 \leq \tau$.

For the general case, i.e., the stationary distribution has nonzero mean $\mu$, we proceed with double conditioning. Specifically, we calculate
\begin{align*}
\EE\left[(z_{k + 2h} - z_{k + h})(z_{k + 2h} - z_{k + h})^\top \big| z_k\right] = \EE \left[ \EE \left[(z_{k + 2h} - z_{k + h})(z_{k + 2h} - z_{k + h})^\top \big| z_{k + h}, z_k\right ] \Big | z_k\right].
\end{align*}
Then by the Markov property, the inner expectation is equal to $\EE \left[(z_{k + 2h} - z_{k + h})(z_{k + 2h} - z_{k + h})^\top \big| z_{k + h}\right ]$. By a similar reasoning to the zero mean case, we first calculate the conditional expectation $$\EE \left[(z_{k + 2h} - z_{k + h})(z_{k + 2h} - z_{k + h})^\top \big| z_{k + h}\right ] = \Sigma + \mu \mu^\top - \mu z_{k+h}^\top - z_{k+h} \mu^\top + z_{k+h}z_{k+h}^\top + W,$$ where the remainder $W$ satisfies $\lVert W \rVert_2 = O(\rho^{h/3})$. Then taking expectation conditioning on $z_k$, we can derive
\begin{align*}
\EE \left[\frac{1}{2}(z_{k + 2h} - z_{k + h})(z_{k + 2h} - z_{k + h})^\top \Big| z_k\right ] = \Sigma + E \Sigma ~~\textrm{with}~~\lVert E \rVert_2 = O\left(e^{h\kappa_\rho}\right).
\end{align*}
The calculation is a repetition of the zero mean case with the extra mean term $\mu$.
\end{proof}

\section{Detailed Proofs in Section \ref{theory}}
\subsection{Proof of Lemma \ref{upperbound}}\label{upperbound-proof}
\begin{proof}
We omit the time indicator $t$. Since $\overline{U}$ and $E_r$ has orthonormal columns, we have $\left \lVert E_r^\top \overline{U} \right \rVert_2 \leq 1$. Thus
\begin{align*}
\gamma_i(t) = \left\lVert e_i^\top \overline{U} \right\rVert_2 \leq \left\lVert e_i^\top \overline{U} (E_r^\top \overline{U})^{-1} (E_r^\top \overline{U}) \right\rVert_2 \leq \left\lVert e_i^\top \overline{U} (E_r^\top \overline{U})^{-1} \right\rVert_2 \left\lVert E_r^\top \overline{U} \right\rVert_2 \leq \left\lVert e_i^\top \overline{U} (E_r^\top \overline{U})^{-1} \right\rVert_2 = \tilde{\gamma}_i(t)
\end{align*}
\end{proof}

\subsection{Proof of Theorem \ref{stage2ode}}\label{proof4.5}
\begin{proof}
Compute the infinitesimal increments of $\tilde{\gamma}_i^2(t)$, which is defined to be $$\Delta \tilde{\gamma}^2_{i, s}(t) = \tilde{\gamma}_{i, s+1}^2 - \tilde{\gamma}_{i, s}^2.$$
The sequence $\{z_{sh}, \overline{U}_s\}$ forms a Markov chain. By \textit{Corollary 4.2 of chapter 7.4} of \cite{ethier2009markov}, once 
\begin{align*}
& b_i = \lim_{\eta \rightarrow 0} \E \left [\frac{\Delta \tilde{\gamma}_i^2(t)}{\eta} \bigg| \overline{U}_s, z_{sh}\right] < \infty, \\
& \sigma_i^2 = \lim_{\eta \rightarrow 0} \E \left[\frac{[\Delta \tilde{\gamma}_i^2(t)]^2}{\eta} \bigg| \overline{U}_s, z_{sh}\right] = 0,
\end{align*}
the sequence $\tilde{\gamma}_{i, s}^2(t)$ weakly converges to the solution of the following ODE,
\begin{align*}
d\tilde{\gamma}_i^2 = b_i \tilde{\gamma}_i^2 dt.
\end{align*}
Hence, we must find the mean and variance of $\Delta \tilde{\gamma}_{i, s}^2(t)$. For simplicity, we omit the subscript $s$.
\begin{align*}
\Delta \tilde{\gamma}^2_{i, s} & = e_i^\top (\overline{U} + \Delta \overline{U}) (E_r^\top (\overline{U} + \Delta \overline{U}))^{-1} (E_r^\top (\overline{U} + \Delta \overline{U}))^{-\top} (\overline{U} + \Delta \overline{U})^\top e_i - e_i^\top \overline{U} (E_r^\top \overline{U})^{-1} (E_r^\top \overline{U})^{-\top} \overline{U}^\top e_i \\
& = 2 e_i^\top \Delta \overline{U} (E_r^\top \overline{U})^{-\top} (E_r^\top \overline{U})^{-1} \overline{U}^\top e_i - 2 e_i^\top \overline{U} (E_r^\top \overline{U})^{-1} (E_r^\top \overline{U})^{-\top} (E_r^\top \Delta \overline{U})^{\top} (E_r^\top \overline{U})^{-\top} \overline{U}^\top e_i + O(\lVert \Delta \overline{U} \rVert_2^2) \notag\\
& = 2 \eta e_i^\top \overline{U} (E_r^\top \overline{U})^{-1} (E_r^\top \overline{U})^{-\top} \overline{U}^\top \overline{X} e_i - 2 \eta e_i^\top \overline{U} (E_r^\top \overline{U})^{-1} (E_r^\top \overline{U})^{-\top} (E_r^\top X \overline{U})^\top (E_r^\top \overline{U})^{-\top} \overline{U}^\top e_i + O(\eta^2 \lVert \overline{X} \rVert^2_2) \notag
\end{align*}
where $\Delta \overline{U} = \eta (I - \overline{U}\overline{U}^\top) \overline{X} \overline{U} + O(\eta^2 \overline{X}^2)$. We have used the fact that 
\begin{align*}
(E_r^\top (\overline{U} + \Delta \overline{U}))^{-1} & = ((E_r^\top \overline{U})(I + (E_r^\top \overline{U})^{-1} (E_r^\top \Delta \overline{U})))^{-1} \\
& = (I - (E_r^\top \overline{U})^{-1} (E_r^\top \Delta \overline{U}) + O(\Delta \overline{U}^2))(E_r^\top \overline{U})^{-1}.
\end{align*}
We only assume $\EE[\lVert \overline{X} \rVert_2^2] < \infty$ without assuming $\overline{X}$ is bounded. Thus, in order to take expectation over $\overline{X}$, we need a truncation argument. Write the SVD of $\overline{X}$ as $\overline{X} = V^\top S V$. Then $\overline{X}_n = V^\top (S \wedge n) V$ denotes the truncated $\overline{X}$ where $a \wedge b = \min(a, b)$ and $S \wedge n$ means to perform such an operation on each diagonal elements of $S$. Clearly, $\overline{X}_n$ has bounded norm $\lVert X_n \rVert_2 \leq n$. Thus, we can take expectation with this truncated random variable $\overline{X}_n$. Moreover, as $n$ increases, $\lVert \overline{X}_n \rVert_2$ also monotone increases to $\lVert \overline{X}(t) \rVert_2$. Then by the monotone convergence theorem, $\lim_{n \rightarrow \infty} \EE[\lVert \overline{X}_n \rVert^2_2] = \EE[ \lVert \overline{X} \rVert^2_2]$. This result allows us to take expectation on the infinitesimal increments $\Delta \tilde{\gamma}_{i, s}^2(t)$.

Taking expectation conditioning on $\overline{U}_s$ and $z_{sh}$, then dividing both sides by $\eta$, we have
\begin{align*}
\E [\frac{\Delta \gamma^2_{i, s}}{\eta} | \overline{U}_s, z_{sh}] & =  2 e_i^\top \overline{U} (E_r^\top \overline{U})^{-1} (E_r^\top \overline{U})^{-\top} \overline{U}^\top (I + E) \Lambda e_i \notag \\
& ~~- 2 e_i^\top \overline{U} (E_r^\top \overline{U})^{-1} (E_r^\top \overline{U})^{-\top} (E_r^\top (I + E) \Lambda \overline{U})^\top (E_r^\top \overline{U})^{-\top} \overline{U}^\top e_i + O(\eta \lVert \Lambda \rVert^2_2) \\
& = 2 e_i^\top \overline{U} (E_r^\top \overline{U})^{-1} (E_r^\top \overline{U})^{-\top} \overline{U}^\top \Lambda e_i + 2 e_i^\top \overline{U} (E_r^\top \overline{U})^{-1} (E_r^\top \overline{U})^{-\top} \overline{U}^\top E \Lambda e_i \notag\\
& ~~-2 e_i^\top \overline{U} (E_r^\top \overline{U})^{-1} (E_r^\top \overline{U})^{-\top} (E_r^\top \Lambda \overline{U})^\top (E_r^\top \overline{U})^{-\top} \overline{U}^\top e_i \notag \\
& ~~- 2 e_i^\top \overline{U} (E_r^\top \overline{U})^{-1} (E_r^\top \overline{U})^{-\top} (E_r^\top E \Lambda \overline{U})^\top (E_r^\top \overline{U})^{-\top} \overline{U}^\top e_i + O(\eta \lVert \Lambda \rVert^2_2) \notag \\
& =  2 \sigma_i \tilde{\gamma}_i^2(t) - 2 e_i^\top \overline{U} (E_r^\top \overline{U})^{-1}  \Lambda_r (E_r^\top \overline{U})^{-\top} \overline{U}^\top e_i \notag \\
& ~~+ 2 e_i^\top \overline{U} (E_r^\top \overline{U})^{-1} (E_r^\top \overline{U})^{-\top} \overline{U}^\top E \Lambda e_i - 2 e_i^\top \overline{U} (E_r^\top \overline{U})^{-1} (E_r^\top \overline{U})^{-\top} (E_r^\top E \Lambda \overline{U})^\top (E_r^\top \overline{U})^{-\top} \overline{U}^\top e_i \notag \\
& ~~+ O(\eta \lVert \Lambda \rVert^2_2). \notag
\end{align*}
Under the geometric ergodicity condition, we know $\lVert E \rVert_2 \leq \tau$, which implies $\lVert E \Lambda \rVert_2 \leq \tau \sigma_1$. Then we have
\begin{align*}
& e_i^\top \overline{U} (E_r^\top \overline{U})^{-1} \Lambda_r (E_r^\top \overline{U})^{-\top} \overline{U}^\top e_i \geq \lambda_r \tilde{\gamma}_{i, s}^2, \\
&\left |e_i^\top \overline{U} (E_r^\top \overline{U})^{-1} (E_r^\top \overline{U})^{-\top} \overline{U}^\top E \Lambda e_i \right| = O(\eta), \notag \\
& e_i^\top \overline{U} (E_r^\top \overline{U})^{-1} (E_r^\top \overline{U})^{-\top} (E_r^\top E \Lambda \overline{U})^\top (E_r^\top \overline{U})^{-\top} \overline{U}^\top e_i = O(\eta).
\end{align*}
Combining the above three bounds, we have
\begin{align*}
\lim_{\eta \rightarrow 0} \E [\frac{\Delta \gamma^2_{i, s}}{\eta} | \overline{U}_s, z_{sh}] \leq 2(\lambda_i - \lambda_r)\gamma_{i, s}^2.
\end{align*}
This upper bound also implies that $\lim_{\eta \rightarrow 0} \E [\frac{[\Delta \gamma_{i, s}^2]^2}{\eta} | \overline{U}_s, z_{sh}] = 0$, since the numerator is of order $O(\eta^2)$. Thus, we can show $\tilde{\gamma}_{i, \eta}^2(t)$ converges weakly to the solution of
\begin{align*}
d\tilde{\gamma}_i^2 = b_i \tilde{\gamma}_i^2 dt ~~\textrm{with}~~ b_i \leq 2(\lambda_i - \lambda_r).
\end{align*}
\end{proof}

\subsection{Proof of Theorem \ref{stage1sde}}
\begin{proof}\label{sdeproof}
We need the following lemma to bound the smallest eigenvalue of $\overline{U}^\top \Lambda \overline{U}$. Denote by $E_{\cA_r} = [e_{a_1}, e_{a_2}, \dots, e_{a_r}] \in \mathbb{R}^{m \times r}$ where $\cA_r = \{a_1, a_2, \dots, a_r\}$ denotes an index set of $\{1, 2, \dots, m\}$ such that $a_1 > a_2 > \dots > a_r$. Further denote by $\overline{\cA_r}$ the complement of $\cA_r$ in $\{1, 2, \dots, m\}$ and write $\overline{E}_{\cA_r} = E_{\overline{\cA}_r}$. Additionally, write $\Lambda_{\cA_r} = \textrm{diag}(\lambda_{a_1}, \lambda_{a_2}, \dots, \lambda_{a_r})$ and $\overline{\Lambda}_{\cA_r} = \Lambda_{\overline{\cA}_r}$.
\begin{lemma}\label{eigenbound}
Suppose $\lVert \overline{E}_{\cA_r}^\top \overline{U} \rVert^2_{\textrm{F}} \leq O(\delta)$ with $\overline{U} \in \mathbb{R}^{m \times r}$ having orthonormal columns, then
\begin{align*}
\sigma_{\min} (\overline{U}^\top \Lambda \overline{U}) \geq \lambda_{a_r} - O(\delta)
\end{align*}
\end{lemma}
\begin{proof}
Since $\lVert \overline{E}_{\cA_r}^\top \overline{U} \rVert_{\textrm{F}}^2 \leq O(\delta)$, we have $\lVert E_{\cA_r}^\top \overline{U} \rVert^2_2 = r - \lVert \overline{E}_{\cA_r}^\top \overline{U} \rVert_{\textrm{F}}^2 \geq r-O(\delta)$. Therefore, $\lVert \overline{U} e'_{a_i} \rVert_2 \geq 1 - O(\delta)$ for $i \in \{1, \dots, r\}$. Thus, we have
\begin{align*}
e'^\top_j \overline{U}^\top \Lambda \overline{U} e'_j & = e'^\top_j \overline{U}^\top (E_{\cA_r}^\top \Lambda_{\cA_r} E_{\cA_r} + \overline{E}_{\cA_r}^\top \overline{\Lambda} \overline{E}_{\cA_r}) \overline{U} e'_j \\
& \geq \lambda_{a_r} e'^\top_j \overline{U}^\top E_{\cA_r}^\top E_{\cA_r} \overline{U} e'_j + e'^\top_j \overline{U}^\top \overline{E}_{\cA_r}^\top \Lambda \overline{E}_{\cA_r} \overline{U} e'_j \notag \\
& \geq \lambda_{a_r} - O(\delta). \notag
\end{align*}
\end{proof}

Now we turn to the proof of Theorem \ref{stage1sde}. We omit $(t)$ if there is no confusion. Denote by $\Delta \zeta_{ij, \eta}(t) = \zeta_{ij, \eta}(t+\eta) - \zeta_{ij, \eta}(t)$ and $\overline{Z}(t) = \eta^{-1/2} \overline{U}^\eta(t)$. We must show the mean and variance of $\Delta \zeta_{ij, \eta}(t)$ satisfies
\begin{align*}
& K_{ij} = \lim_{\eta \rightarrow 0} \E \left [\frac{\Delta \zeta_{ij, \eta}(t)}{\eta} \bigg| \overline{U}, z_{\lfloor t/\eta \rfloor h + h} \right ] < \infty, \\
& G_{ij}^2 = \lim_{\eta \rightarrow 0} \E \left [\frac{[\Delta \zeta_{ij, \eta}(t)]^2}{\eta} \bigg| \overline{U}, z_{\lfloor t/\eta \rfloor h + h} \right ] < \infty.
\end{align*}
Then the sequence $\zeta_{ij, \eta}(t)$ weakly converges to the solution of the following SDE
\begin{align*}
d\zeta_{ij} = K_{ij} \zeta_{ij} dt + G_{ij} dB_t,
\end{align*}
where $B_t$ is the standard Brownian motion. In fact, we have
\begin{align*}
\E \left [\Delta \zeta_{ij, \eta}(t) | \overline{U}, z_{\lfloor t/\eta \rfloor h + h} \right] & = \E \left [\eta^{-1/2} e'^\top_j Q (\overline{Z}(t+\eta) - \overline{Z}(t))^\top e_i \vert \overline{U}, z_{\lfloor t/\eta \rfloor h + h} \right ] \\
& = \eta^{1/2} e'^\top_j Q \E \left [\overline{U}^\top \overline{X} - \overline{U}^\top \overline{X} \overline{U} \overline{U}^\top \vert \overline{U}, z_{\lfloor t/\eta \rfloor h + h} \right ] e_i + O(\eta^{3/2} \lVert \Lambda \rVert_2^2)\notag\\
& = \eta e'^\top_j Q \overline{Z}^\top (\Lambda + E \Lambda) e_i - \eta e_j'^\top Q (\overline{U}^\top (\Lambda + E \Lambda) \overline{U}) \overline{Z}^\top e_i  + O(\eta^{3/2} \lVert \Lambda \rVert^2_2) \notag\\
& = \eta \lambda_i e'^\top_j Q \overline{Z}^\top e_i - \eta e'^\top_j Q (\overline{U}^\top \Lambda \overline{U}) \overline{Z}^\top e_i + \eta e'^\top_j Q \overline{Z}^\top E \Lambda e_i - \eta e'^\top_j Q (\overline{U}^\top E \Lambda \overline{U}) \overline{Z}^\top e_i \notag \\
&~~+ O(\eta^{3/2} \lVert \Lambda \rVert_2^2) \notag\\
& = \eta \lambda_i \zeta_{ij, \eta} - \eta e'^\top_j Q (\overline{U}^\top \Lambda \overline{U}) \overline{Z}^\top e_i + \eta e'^\top_j Q \overline{Z}^\top E \Lambda e_i - \eta e'^\top_j Q (\overline{U}^\top E \Lambda \overline{U}) \overline{Z}^\top e_i + O(\eta^{3/2} \lVert \Lambda \rVert_2^2). \notag
\end{align*}
By the Lemma, we have $\sigma_{\min}\left( \overline{U}^\top \Lambda \overline{U}\right) \geq \lambda_{a_r} - O(\delta)$. Observe that $\lVert E \Lambda \rVert_2 = O(\eta)$, thus we obtain
\begin{align*}
\lim_{\eta \rightarrow 0} \E \left [\frac{\Delta \zeta_{ij, \eta}(t)}{\eta} \bigg| \overline{U}, z_{\lfloor t/\eta \rfloor h} \right ] = K_{ij} \zeta_{ij, \eta} \textrm{~~with~~} K_{ij} \in [\lambda_i - \lambda_1, \lambda_i - \lambda_{a_r}].
\end{align*}
Note that when $j = r$, we have $K_{ir} = \lambda_i - \lambda_{a_r}$, because the equality $e'^\top_r Q (\overline{U}^\top \Lambda \overline{U}) \overline{Z}^\top e_i = \lambda_{a_r} e_r'^\top Q \overline{Z}^\top e_i + O(\delta)$ holds. The variance is
\begin{align*}
\E \left[[\Delta \zeta_{ij, \eta}(t)]^2 | \overline{U}, z_{\lfloor t/\eta \rfloor h + h} \right ] & = \E \left [\left(\eta^{-1/2} e'^\top_j Q (\overline{Z}(t+\eta) - \overline{Z}(t))^\top e_i\right)^2 \bigg| \overline{U}, z_{\lfloor t/\eta \rfloor h + h} \right ] \\
& = \eta \E \left [\left (e'^\top_j Q \overline{U}^\top \overline{X} (I - \overline{U} \overline{U}^\top) e_i\right )^2 \bigg| \overline{U}, z_{\lfloor t/\eta \rfloor h + h} \right ] + O(\eta^2 \lVert \Lambda \rVert_2^2). \notag
\end{align*}
Observe that we have $\overline{U}^\top (I - \overline{U} \overline{U}^\top) = 0$, therefore, $\overline{U} Q^\top e'_j$ and $(I - \overline{U} \overline{U}^\top) e_i$ are orthogonal. Moreover, the norm of these two vectors satisfies $\left \lVert e'^\top_j Q \overline{U} \right \rVert_2 \approx 1$ and $\left \lVert (I - \overline{U} \overline{U}^\top) e'_j \right \rVert_2 \leq 1$. Hence, by the assumption that $\overline{X}$ has bounded second moment, we have
\begin{align*}
\lim_{\eta \rightarrow 0} \E \left [\frac{[\Delta \zeta_{ij, \eta}(t)]^2}{\eta} \bigg | \overline{U}, z_{\lfloor t/\eta \rfloor h + h} \right ] < \infty.
\end{align*}
\end{proof}

\subsection{Proof of Proposition \ref{saddletime}}\label{saddletime-proof}
\begin{proof}
Since we start the algorithm at the saddle point and $K_{rr} = \lambda_r - \lambda_{r+1}$. The continuous time process $\zeta_{rr}(t)$ is approximately Gaussian distributed with mean 0 and variance $\frac{G_{rr}^2}{2K_{rr}}(e^{2K_{rr}t} - 1)$. We need the following condition,
\begin{align*}
\mathbb{P}\left (\left \lVert e_r^\top \overline{U}(t) \right \rVert_2^2 \geq \delta^2 \right ) \geq \mathbb{P}\left (\zeta^2_{rr}(t) \geq \eta^{-1} \delta^2 \right ) \geq 1 - \nu,
\end{align*}
which is equivalent to
\begin{align*}
\mathbb{P}\left(\zeta^2_{rr}(t) \geq \eta^{-1} \delta^2\right) = \mathbb{P} \left (\frac{|\zeta_{rr}(t)|}{\sqrt{\frac{G_{rr}^2}{2K_{rr}}(e^{2K_{rr}t} - 1)}} \geq \frac{\eta^{-1/2}\delta}{\sqrt{\frac{G_{rr}^2}{2K_{rr}}(e^{2K_{rr}t} - 1)}} \right ).
\end{align*}
Note that $\frac{\zeta_{rr, \eta}(t)}{\sqrt{\frac{G_{rr}^2}{2K_{rr}}(e^{2K_{rr}t} - 1)}}$ converges weakly to $\frac{\zeta_{rr}(t)}{\sqrt{\frac{G_{rr}^2}{2K_{rr}}(e^{2K_{rr}t} - 1)}}$, which is a standard Gaussian random variable. Let $\Phi(\cdot)$ denotes the standard Gaussian CDF, then we have
\begin{align*}
\eta^{-1/2} \delta \leq - \Phi^{-1}\left(\frac{1 - \nu/2}{2}\right) \sqrt{\frac{G_{rr}^2}{2K_{rr}}(e^{2K_{rr}t} - 1)}.
\end{align*}
Rearrange the above terms, we get
\begin{align*}
T_1 = \frac{1}{2(\lambda_r - \lambda_{r+1})} \log\left (\frac{2(\lambda_r - \lambda_{r+1})\eta^{-1}\delta^2}{[\Phi^{-1}(\frac{1-\nu/2}{2})]^2 G_{rr}^2} + 1\right).
\end{align*}
\end{proof}

\subsection{Proof of Proposition \ref{stage2time}}\label{stage2time-proof}
\begin{proof}
We know $\left \lVert \cos \Theta(\overline{E}_r, \overline{U}^\eta(t)) \right \rVert^2_{\textrm{F}} = \sum_{i=r+1}^m \gamma_{i, \eta}^2(t)$. Then using the upper bound $\tilde{\gamma}_i^2(t)$, we have
\begin{align*}
\left\lVert \cos \Theta(\overline{E}_r, \overline{U}^\eta(t)) \right\rVert^2_{\textrm{F}} = \sum_{i=r+1}^m \gamma_{i, \eta}^2(t) \leq \sum_{i=r+1}^m \tilde{\gamma}_{i, \eta}^2(t) = \sum_{i=r+1}^m \tilde{\gamma}_i^2(0) e^{b_i t} \leq \sum_{i=r+1}^m \tilde{\gamma}_i^2(0) e^{2(\lambda_r - \lambda_{r+1}) t}.
\end{align*}
In order for $\left \lVert \cos \Theta(\overline{E}_r, \overline{U}^\eta(T_2)) \right \rVert^2_{\textrm{F}} \leq \delta^2$, we need at most $T_2$ time such that
\begin{align*}
\sum_{i=r+1}^m \tilde{\gamma}_i^2(0) e^{2(\lambda_{r+1} - \lambda_r) T_2} \leq \delta^2.
\end{align*}
Since the algorithm has escaped from the saddle point, we have $\left \lVert e_{r+1}^\top \overline{U} \right \rVert_2^2 \leq 1 - \delta^2$ and $\left \lVert E_r^\top \overline{U} \right \rVert_2^2 \geq \delta^2$. Thus, the initial value satisfies $\sum_{i=r+1}^m \tilde{\gamma}_i^2(0) \leq (1-\delta^2)\delta^{-2}$. Taking logarithm on both sides yields
\begin{align*}
T_2 = \frac{1}{2(\lambda_r - \lambda_{r+1})} \log \frac{1-\delta^2}{\delta^4} = \frac{1}{\lambda_r - \lambda_{r+1}} \log \frac{\sqrt{1-\delta^2}}{\delta^2}.
\end{align*}
Then for a sufficiently small $\eta$, we have
\begin{align}
\PP\left(\sum_{i=r+1}^m \gamma_{i, \eta}^2(T_2) \leq \delta^2\right) \geq \frac{3}{4}, \notag
\end{align}
with $T_2 \asymp \frac{1}{\lambda_r - \lambda_{r+1}} \log \frac{1}{\delta^2}$.
\end{proof}

\subsection{Proof of Theorem \ref{stage3sde-thm}}\label{stage3sde-thm-proof}
\begin{proof}
The technique is almost the same as in Theorem \ref{stage1sde}. We have
\begin{align*}
\E \left [\Delta \zeta_{ij, \eta}(t) | \overline{U}, z_{\lfloor t/\eta \rfloor h + h} \right ] & = \E\left [\eta^{-1/2} e'^\top_j Q (\overline{Z}(t+\eta) - \overline{Z}(t))^\top e_i | \overline{U}, z_{\lfloor t/\eta \rfloor h + h}\right] \\
& = \eta \sigma_i \zeta_{ij, \eta} - \eta e'^\top_j Q (\overline{U}^\top \Lambda \overline{U}) \overline{Z}^\top e_i + \eta e'^\top_j Q \overline{Z}^\top E \Lambda e_i - \eta e'^\top_j Q (\overline{U}^\top E \Lambda \overline{U}) \overline{Z}^\top e_i + O(\eta^{3/2} \lVert \Lambda \rVert_2^2), \notag
\end{align*}
and the variance satisfies
\begin{align*}
\E \left [(\Delta \zeta_{ij, \eta}(t))^2 | \overline{U}(t), z_{\lfloor t/\eta \rfloor h + h} \right] & = \E\left [\eta^{-1/2} e'^\top_j Q (\overline{Z}(t+\eta) - \overline{Z}(t))^\top e_i | \overline{U}, z_{\lfloor t/\eta \rfloor h + h}\right] \\
& = \eta \E\left [(e'^\top_j Q \overline{U}^\top \overline{X} (I - \overline{U} \overline{U}^\top) e_i)^2 | \overline{U}, z_{\lfloor s/\eta \rfloor h + h}\right] + O(\eta^2 \lVert \Lambda \rVert_2^2). \notag
\end{align*}
Thus, with $\sigma_{\min} (\overline{U}^\top \Lambda \overline{U}) \geq \lambda_r - O(\delta)$ by Lemma \ref{eigenbound}, we have
\begin{align*}
& \lim_{\eta \rightarrow 0} \E\left [\frac{\Delta \zeta_{ij, \eta}(t)}{\eta} | \overline{U}, z_{\lfloor t/\eta \rfloor h + h}\right] = K_{ij} \zeta_{ij,\eta}(t) \textrm{~~with~~} K_{ij} \in [\lambda_i - \lambda_1, \lambda_i - \lambda_r], \\
& \lim_{\eta \rightarrow 0} \E \left[\frac{[\Delta \zeta_{ij, \eta}(t)]^2}{\eta} | \overline{U}, z_{\lfloor t/\eta \rfloor h + h} \right] < \infty.
\end{align*}
\end{proof}

\subsection{Proof of Proposition \ref{stage3time}}\label{stage3time-proof}
\begin{proof}
The proof is an application of Markov's inequality. Observe again that $\left \lVert \cos \Theta(\overline{E}_r, \overline{U}(t)) \right \rVert^2_{\textrm{F}} = \eta \sum_{i=r+1}^m \sum_{j=1}^r \zeta_{ij}^2(t)$. The expectation of $\zeta_{ij}^2(t)$ can be found as follows,
\begin{align*}
\E[\zeta_{ij}^2(t)] & = \zeta_{ij}^2(0) e^{2K_{ij}t} + \frac{G_{ij}^2}{2K_{ij}}(e^{2K_{ij}t} - 1) \\
& \leq \zeta_{ij}^2(0) e^{2(\lambda_{r+1} - \lambda_r)t} + \frac{G_{ij}^2}{2(\lambda_r - \lambda_{r+1})}. \notag
\end{align*}
By Markov's inequality, we have
\begin{align*}
\mathbb{P}\left (\left \lVert \cos \Theta(\overline{E}_r, \overline{U}(t)) \right \rVert^2_{\textrm{F}} > \epsilon \right) & \leq \frac{\E\left[\eta \sum_{i=r+1}^m \sum_{j=1}^r \zeta_{ij}^2(t)\right]}{\epsilon} \\
& \leq \frac{\eta}{\epsilon} \sum_{i=r+1}^m \sum_{j=1}^r \zeta_{ij}^2(0) e^{2(\lambda_{r+1} - \lambda_r)t} + \frac{\eta}{\epsilon} r\frac{G_m}{2(\lambda_r - \lambda_{r + 1})}. \notag \\
\end{align*}
Note that $\left \lVert \cos \Theta(\overline{E}_r, \overline{U}^\eta(t)) \right \rVert^2_{\textrm{F}}$ weakly converges to $\left \lVert \cos \Theta(\overline{E}_r, \overline{U}(t)) \right \rVert^2_{\textrm{F}}$, then we need at most $T_3$ time satisfying
\begin{align*}
\frac{\eta}{\epsilon} \sum_{i=r+1}^m \sum_{j=1}^r \zeta_{ij}^2(0) e^{2(\lambda_{r+1} - \lambda_r)T_3} + \frac{\eta}{\epsilon} r\frac{G_m}{2(\lambda_r - \lambda_{r + 1})} \leq \frac{1}{8}.
\end{align*}
Rearrange and combine with $\eta \sum_{i=r+1}^m \sum_{j=1}^r \zeta_{ij}^2(0) \leq \delta^2$, we get
\begin{align*}
T_3 = \frac{1}{2(\lambda_r - \lambda_{r+1})} \log \left (\frac{8(\lambda_r - \lambda_{r+1})\delta^2}{(\lambda_r - \lambda_{r+1})\epsilon - 4 \eta r G_m} \right).
\end{align*}
\end{proof}

\subsection{Proof of Corollary \ref{totalcor}}\label{totalcor-proof}
\begin{proof}
We list the time upper bound given in the Stage 1, Stage 2 and Stage 3, 
\begin{align*}
& T_1 = \frac{1}{2(\lambda_r - \lambda_{r+1})} \log\left(\frac{2(\lambda_r - \lambda_{r+1})\eta^{-1}\delta^2}{[\Phi^{-1}(\frac{1-\nu/2}{2})]^2 G_{rr}^2} + 1\right), \\
& T_2 = \frac{1}{\lambda_r - \lambda_{r+1}} \log \frac{1}{\delta^2}, \notag \\
& T_3 = \frac{1}{2(\lambda_r - \lambda_{r+1})} \log \left (\frac{8(\lambda_r - \lambda_{r+1})\delta^2}{(\lambda_r - \lambda_{r+1})\epsilon - 4 \eta r G_m} \right). \notag
\end{align*}
Choose the step size $\eta$ satisfies
\begin{align*}
\eta \asymp \frac{(\lambda_r - \lambda_{r+1})\epsilon}{5rG_m}.
\end{align*}
With such a choice of $\eta$ and $\delta = O(\eta^{1/2})$, we have
\begin{align*}
\log \left (\frac{(\lambda_r - \lambda_{r+1})\delta^2}{(\lambda_r - \lambda_{r+1})\epsilon - 4 \eta r G_m} \right) \asymp \log \frac{\lambda_r - \lambda_{r+1}}{rG_m}. 
\end{align*}
The total time $T$ is upper bounded by
\begin{align*}
T & = T_1 + T_2 + T_3 \\
& = \frac{1}{2(\lambda_r - \lambda_{r+1})} \log\left(\frac{2(\lambda_r - \lambda_{r+1})\eta^{-1}\delta^2}{[\Phi^{-1}(\frac{1-\nu/2}{2})]^2 G_{rr}^2} + 1\right) + \frac{1}{\lambda_r - \lambda_{r+1}} \log \frac{1}{\delta^2} \notag \\
& ~~+ \frac{1}{2(\lambda_r - \lambda_{r+1})}  \log \left (\frac{8(\lambda_r - \lambda_{r+1})\delta^2}{(\lambda_r - \lambda_{r+1})\epsilon - 4 \eta r G_m} \right) \notag \\
& \asymp \frac{1}{\lambda_r - \lambda_{r+1}} + \frac{1}{\lambda_r - \lambda_{r+1}} \log \frac{rG_m}{\epsilon(\lambda_r - \lambda_{r+1})} + \frac{1}{\lambda_r - \lambda_{r+1}} \log \frac{\lambda_r - \lambda_{r+1}}{rG_m} \\
& \asymp \frac{1}{\lambda_r - \lambda_{r+1}} \log \frac{rG_m}{\epsilon(\lambda_r - \lambda_{r+1})}.
\end{align*}
\end{proof}

\end{document}